\documentclass{article}

\usepackage[final]{neurips_2022}

\usepackage[utf8]{inputenc}
\usepackage[T1]{fontenc}
\usepackage{hyperref}
\usepackage{url}
\usepackage{booktabs}
\usepackage{amsfonts}
\usepackage{nicefrac}
\usepackage{microtype}
\usepackage{xcolor}
\usepackage{amssymb}
\usepackage{amsmath}
\usepackage{mathrsfs}
\usepackage{amsthm}
\usepackage{algorithm}
\usepackage{algorithmicx}
\usepackage{algpseudocode}
\usepackage{mathtools}
\usepackage{multirow}
\usepackage{wrapfig}
\usepackage{tabto}
\usepackage{subfigure}
\usepackage{listings}

\usepackage{tikz}
\usetikzlibrary{patterns}
\usetikzlibrary{calc}

\usepackage{xcolor}
\definecolor{darkteal}{HTML}{1697B7}
\definecolor{lightteal}{HTML}{a2d0db}
\definecolor{darkred}{HTML}{C23000}
\definecolor{darkblue}{rgb}{0,0.08,0.45}
\definecolor{algann}{HTML}{0030C2}
\definecolor{Green}{rgb}{0.8509803921568627, 1.0, 0.8392156862745098}
\definecolor{Cyan}{rgb}{0.87, 0.96, 1}
\hypersetup{
    colorlinks=true,
    citecolor=darkteal,
    linkcolor=darkred,
    urlcolor=darkteal
}

\newcommand\tens[1]{\mathcal{#1}}
\newcommand{\convop}{\mathscr{C}}
\newcommand{\convopone}{\mathscr{T}}

\newcommand{\cin}{c_{\mathrm{in}}}
\newcommand{\cout}{c_{\mathrm{out}}}
\newcommand{\ch}{c}
\renewcommand{\dim}{d}
\newcommand{\vecs}{\texttt{vec}}
\newcommand{\Tmat}{{T}}
\newcommand{\stride}{s}
\DeclareMathSymbol{\shortminus}{\mathbin}{AMSa}{"39}

\usepackage{tabularx}
\usepackage{color, colortbl}

\newtheorem{theorem}{Theorem}

\newcolumntype{x}[1]{>{\centering\arraybackslash\hspace{0pt}}p{#1}}
\newcommand{\fixedvspace}[1]{%
  \par\kern-\prevdepth\vspace{#1}%
}

\newcommand{\minbox}[2]{\mathmakebox[\ifdim#1<\width\width\else#1\fi]{#2}}
\newcommand{\Let}[2]{\State $ \minbox{1em}{#1} \gets #2 $}
\algnewcommand{\Local}{\State\textbf{local: }}
\algnewcommand{\CommentGray}[1]{\hfill\textcolor{gray}{$\triangleright$\,#1}}
\algrenewcommand\algorithmicindent{0.9em}
\algblock{ParFor}{EndParFor}
\algnewcommand\algorithmicparfor{\textbf{parfor}}
\algnewcommand\algorithmicpardo{\textbf{do}}
\algnewcommand\algorithmicendparfor{\textbf{end\ parfor}}
\algrenewtext{ParFor}[1]{\algorithmicparfor\ #1\ \algorithmicpardo}
\algrenewtext{EndParFor}{\algorithmicendparfor}

\newcommand\blfootnote[1]{%
  \begingroup
  \renewcommand\thefootnote{}\footnote{#1}%
  \addtocounter{footnote}{-1}%
  \endgroup
}

\title{Towards Practical Control of Singular Values of Convolutional Layers}

\author{%
   Alexandra Senderovich$^{*\dagger}$ \\
   HSE University
   \And
   Ekaterina Bulatova$^*$ \\
   HSE University
   \And
   Anton Obukhov \\
   ETH Zürich \\
   \And
   Maxim Rakhuba \\
   HSE University
}

\begin{document}

\maketitle

\begin{abstract}
In general, convolutional neural networks (CNNs) are easy to train, but their essential properties, such as generalization error and adversarial robustness, are hard to control.
Recent research demonstrated that singular values of convolutional layers significantly affect such elusive properties and offered several methods for controlling them.
Nevertheless, these methods present an intractable computational challenge or resort to coarse approximations. 
In this paper, we offer a principled approach to alleviating constraints of the prior art at the expense of an insignificant reduction in layer expressivity.
Our method is based on the tensor-train decomposition; it retains control over the actual singular values of convolutional mappings while providing structurally sparse and hardware-friendly representation.
We demonstrate the improved properties of modern CNNs with our method and analyze its impact on the model performance, calibration, and adversarial robustness.
The source code is available at:
{
\href{https://github.com/WhiteTeaDragon/practical_svd_conv}{
https://github.com/WhiteTeaDragon/practical\_svd\_conv
}
}
\end{abstract}

\section{Introduction}

\blfootnote{$^*$Equal contribution.}
\blfootnote{$^\dagger$Corresponding author: Alexandra Senderovich (\href{mailto:AlexandraSenderovich@gmail.com}{AlexandraSenderovich@gmail.com})}%

Over the past decade, empirical advances in Deep Learning have made machine learning ubiquitous to researchers and practitioners from various fields of science and industry. 
However, the theory of Deep Learning lags behind its practical applications, resulting in unexpected outcomes or lack of explainability of the models. 
These adverse effects are becoming more prominent with many models put into customer-facing products, such as perception systems of self-driving cars, chatbots, and other applications. 
The demand for tighter control over deployed models has given rise to several subfields of Deep Learning, such as the study of adversarial robustness and model calibration, to name a few. 

Convolutional neural networks (CNNs) and, in particular, residual CNNs have become a benchmark for various computer vision tasks.
The key component of CNNs is convolutional layers, representing linear transformations that account for the image data structure.
The singular values of these linear transformations are key to the properties of the whole network, such as the Lipschitz constant and, hence, generalization error and robustness to adversarial examples.
Moreover, bounding the Lipschitz constant of convolutional layers can also improve training, as it prevents gradients from exploding.

Nevertheless, finding and controlling singular values of convolutional layers is challenging.
Indeed, computing and then clipping exact singular values of a layer~\citep{sedghi2019singular}, given by a kernel tensor of the size $k\times k \times \cin \times \cout$, has time and space complexities of $\mathcal{O}(n^2 \ch^2(\ch+\log n))$ and $\mathcal{O}(n^2{\ch}^2)$ respectively, where $n\times n$ is the input image size, $\cin$ and $\cout$ are the numbers of input and output channels respectively, $\ch=\max\{\cin,\cout\}$, and $k\times k$ is the filter size. 
Since typically $k^2 \ll \ch$, computing singular values requires many more operations than computing a single convolution with its asymptotic complexity $\mathcal{O}(n^2\ch^2 k^2)$.
When computing singular values, the storage of an arising kernel tensor of the size $n\times n \times \cin \times \cout$ (padded to image dimensions) can also be an issue for larger networks and inputs. 
This effect becomes even more pronounced in three-dimensional convolutions used for volumetric data processing.
Alternative methods that parametrize the convolutional layers are also computationally demanding~\citep{singla2021skew,BCOP}.

In this paper, we propose a practical approach to constraining the singular values  of the convolutional layers based on intractable techniques from the prior art.
Our approach relies on the assumption that modern over-parameterized neural networks can be made sparse using tensor decompositions, incurring insignificant degradation of the downstream task performance. We investigate the impact of using our approach on model performance, calibration, and adversarial robustness.
More specifically, our contributions are as follows:
\begin{enumerate}
    \item We propose a new framework for reducing the computational complexity of controlling singular values of a convolutional layer by imposing the Tensor-Train (TT) decomposition constraint on a kernel tensor.
    It allows for substituting the computation of singular values of the original layer by a smaller layer. As a result, it reduces the complexity of singular values control by using a method of choice.
    \item We extend the formula of exact singular values of a convolutional layer~\citep{sedghi2019singular} to the case of strided convolutions, which is ubiquitous in CNN architectures.
    \item
    We apply the proposed framework to different methods of controlling singular values and test them  on several CNN architectures. Contrary to measuring just the downstream task performance as in the prior art, we additionally demonstrate improvements in adversarial robustness and model calibration of the networks.
\end{enumerate}

The paper is organized as follows: Sec.~\ref{sec:relwork} provides an overview of prior art on singular values and Lipschitz constant estimation, related methods employing these techniques, and assumptions used to develop our approach; Sec.~\ref{sec:singularvalues} revisits computation of singular values of convolutional layers and describes our approach to dealing with computational complexity;
Sec.~\ref{sec:method} walks through the steps of our algorithm to computing singular values of compressed convolutional layers;
Sec.~\ref{sec:experiments} contains the empirical study of our method; 
Sec.~\ref{sec:conclusion} wraps up the paper.

\section{Related Work}
\label{sec:relwork}

\subsection{Computing Singular Values of Convolutions}
\label{sec:relwork:twoperspectives}
Traditional discrete convolutional layers used in signal processing and computer vision~\citep{lecun} can be seen from two perspectives.
On the one hand, a $d$-dimensional convolution linearly maps (potentially overlapping) windows of a $c$-channel input signal with side $k$ into vectors of $c$ output values (\textit{window map}).
This perspective is related to the one implementation of convolutional layers, decomposing the operation into window extraction termed \texttt{im2col}~\citep{chellapilla2006high}, followed by matrix multiplication via \texttt{GEMM}~\citep{blackford2002updated}.
From this perspective, the transformation is defined by the convolution weight tensor unfolded into a matrix of size $k^2\cin \times \cout$ (\texttt{Conv2D} case); 
its singular values can be computed and controlled through SVD-based schemes over the said unfolding matrix.
For example, \citet{yoshida2017spectral} proposed a method of Spectral Normalization to stabilize GAN training~\citep{gan}. 
Their method effectively estimates the largest singular value of the map via the power method. 
In a similar vein, \citet{liu2019spectral,obukhov2021spectral} explicitly parametrized the SVD decomposition of the mapping to control multiple singular values during training.

Nevertheless, the above perspective lacks the generality of treating the input signal as a whole, which is especially important when chaining maps, such as seen in deep CNNs.
To this end, a discrete convolutional layer can be viewed as a mapping from and onto the space defined by the signal shape (\textit{signal map}). 
{The map is defined by an $N$-ly block-circulant matrix, where $N$ is the number of spatial dimensions of the convolution.}
A straightforward way to compute singular values of such a matrix by applying a full SVD algorithm is prohibitively expensive, even for small numbers of channels, as well as signal and window sizes.
To overcome this issue, \citet{sedghi2019singular} proposed a method based on the Fast Fourier Transform (FFT) for computing exact singular values of convolutional layers that has a much better complexity than the naive approach. Nevertheless, it is still quite demanding in space and time complexities, as it deals with a padded kernel of the size $n\times n \times \ch \times \ch$. 
Alternatively, \citet{singla2021skew} directly impose nearly orthogonal constraints on convolutional layers using Taylor expansion of a matrix exponential of a skew-Hermitian matrix, which appears to be more efficient than a Block Convolution Orthogonal Parameterization approach of~\citet{BCOP} proposed earlier. 
\citet{singla2021fantastic} drew connections between both perspectives and proved that the Lipschitz constant of the window map could serve as a genuine upper bound of the Lipschitz constant of the signal map if multiplied by a certain constant.

\subsection{Effects of Singular Values}
The study of singular values and the closely related Lipschitzness of convolutional neural networks impacts many domains of Deep Learning research.
One essential expectation about neural networks is to generalize input data instead of memorizing it.
\citet{bartlett2017spectrallynormalized} present the generalization bound that depends on the product of the largest singular values of layers' Jacobians.
\citet{gouk2020regularisation} performed a detailed empirical study of the influence of bounding individual Lipschitz constants of each layer and its effect on the generalization error of the whole CNN.
Singular values of Jacobians greater or smaller than $1.0$ can also be responsible for the growth or decay of gradients.
Therefore, controlling all singular values can also be helpful to avoid exploding or vanishing gradients.
The decay of singular values in layers also plays a role in the performance of GAN generators~\citep{liu2019spectral}.

Apart from boosting the accuracy of predictions, researchers have been working on improving the robustness of neural networks to adversarial attacks, which is also affected by the Lipschitz constant~\citep{cisse2017parseval}.
Adversarial attacks aim to find a negligible (to human perception) perturbation of input data that would sabotage the predictions of a model.
In this paper, we use the AutoAttack Robust Accuracy module~\citep{autoattack} to evaluate results.
We additionally use a metric of classification model calibration (Expected Calibration Error, ECE)~\citep{ECEintro}.
Connections between model Lipschitz constraints, calibration, and out-of-distribution (OOD) model performance have been drawn in recent literature~\citep{postels21deu}.
Various training techniques can also affect model properties~\citep{vetrov2022} and efficiency of training \citep{ijcai2022p769}.

\subsection{Tensor Decompositions in Deep Learning}
Our proposed framework for CNN weights reparameterization relies on the well-studied property of neural network overparameterization~\citep{lotterytickethypothesis}.
Many network compression approaches~\citep{structuredpruning,lee2018snip} agree that it is often possible to approach the uncompressed model performance by imposing some form of sparsity on the network weights.
Tensor decompositions have been used previously for compression~\citep{NIPS2015_6855456e,DBLP:journals/corr/GaripovPNV16,DBLP:journals/corr/abs-1802-09052,obukhov2020tbasis}, multitask learning~\citep{KanakisBSGOG20}, neural rendering~\citep{obukhov2022ttnf} and reinforcement learning~\citep{ttopt2022}.
The idea of this paper is to utilize tensor decomposition closely related to the SVD decomposition and to take advantage of its sparsity to reduce the complexity of controlling singular values.
For that, we resort to the TT decomposition~\citep{oseledets2011tensor}, which in the outlined context is also equivalent to the Tucker-2 decomposition~\citep{tucker}.
As opposed to alternative tensor decompositions, this one allows us to naturally access singular values of the convolution layer with any method of choice.

\section{Method-independent Complexity Reduction}
\label{sec:singularvalues}

This section explains how to reduce the complexity of calculating the singular values of a convolutional layer (in the \textit{signal map} sense explained in Sec.~\ref{sec:relwork:twoperspectives}), regardless of the chosen method.
For the sake of generality, we consider layers acting on $(d{+}1)$-dimensional input arrays (tensors) with $\dim$ spatial dimensions and $1$ channel dimension.
Of particular interest are cases $\dim=2$ and $\dim=3$, corresponding to regular images and volumetric data.

We bootstrap our approach based on the recent research on neural network sparsity by compressing convolutional layers with the TT decomposition. 
We focus on dealing with the channel-wise complexity of the methods employed. 
In what follows,
Sec.~\ref{sec:convlayers} revisits the convolutional operator in neural networks; Sec.~\ref{sec:method_main} introduces our principled approach to compressing convolutional layers and reducing the complexity of the problem by a significant margin. 

\subsection{Regular Convolutional Layers}
\label{sec:convlayers}

To introduce a convolutional layer formally, let $\tens{K}\in\mathbb{R}^{k\times \dots \times k \times \cin \times \cout}$ be a $(\dim{+}2)$-dimensional kernel tensor, where $k$ is a filter size and $\cin,\cout$ are the numbers of input and output channels, respectively. 
A convolutional layer is given by a linear map $\convop_{\tens{K}}$ -- multichannel convolution with the kernel tensor $\tens{K}$ such that $\convop_{\tens{K}}\colon  \mathbb{R}^{\cin \times n\times \dots \times n} \to\mathbb{R}^{\cout \times f(n,k)\times \dots \times f(n,k)}$ and
\begin{equation}
\label{eq:assumption}
    \left(\convop_{\tens{K}}(\tens{X})\right)_{j,:,\dots,:} = \sum_{i=0}^{\cin-1} \tens{K}_{:,\dots,:,i,j} * \tens{X}_{i,:,\dots,:}, \quad\  j = 1,\dots,\cout,
\end{equation}%
where $f(n,k)$ is an integer, depending on the type of the convolution $'*'$ used and convolution parameters, e.g., strides.
The key assumption needed for our derivations is the bilinearity of $'*'$, which covers different convolution types, e.g., linear, periodic, and correlation.
For example, a widely used correlation-type convolution with strides equal to one, reads
\[
    \tens{Y}_{jq_1\dots q_\dim} = \sum_{i=0}^{\cin-1} \sum_{p_1,\dots,p_\dim=0}^{k-1} \,  \tens{K}_{p_1\dots p_\dim ij} \tens{X}_{i, q_1 + p_1,\dots, q_\dim  + p_\dim},
\]
for all $j=0,\dots, \cout-1$ and $q_\alpha = 0,\dots, f(n,k)-1$, $\alpha=1,\dots,\dim$, where $f(n,k) = n-k+1$.
In what follows in this section, we write
\begin{equation} \label{eq:conv_main}
    \tens{Y} = \convop_{\tens{K}}(\tens{X}),
\end{equation}
implying one of the convolution types mentioned above.

Using the linearity of~$\convop_{\tens{K}}$, we can rewrite~\eqref{eq:conv_main} as a matrix-vector product:
\begin{equation} \label{eq:conv_vec}
    \vecs (\tens{Y}) = \Tmat_{\tens{K}}\, \vecs (\tens{X}),
\end{equation}
where $\vecs$ is a row-major reshaping of a multidimensional array into a column vector. In turn, $\Tmat_{\tens{K}} \in\mathbb{R}^{\cout n^\dim \times \cin n^\dim}$ is a matrix with an additional block structure: its $n^\dim \times n^\dim$ blocks are $\dim$-level Toeplitz matrices (see~\citet{sedghi2019singular} for $\dim=2$).
Representation~\eqref{eq:conv_vec} allows us to replace singular values of a linear map $\convop_\tens{K}$ with singular values of a matrix~$\Tmat_{\tens{K}}$, which is handy for analysis.
In turn, access to singular values allows for controlling the Lipschitz constant of $\convop_{\tens{K}}$ in terms of the largest singular value of $T_{\tens{K}}$, denoted as $\sigma_1(T_{\tens{K}})$.
Indeed, since $\sigma_1(\Tmat_{\tens{K}}) = \|\Tmat_{\tens{K}}\|_2$, we have
\[
\begin{split}
\|\convop_{\tens{K}}(\tens{X}) - \convop_{\tens{K}}(\tens{Z})\|_F = \|\Tmat_{\tens{K}}(\vecs(\tens{X}) - \vecs(\tens{Z}))\|_2 \leq \|\Tmat_{\tens{K}}\|_2 \|\tens{X} - \tens{Z}\|_F =  \sigma_1(\Tmat_{\tens{K}}) \|\tens{X} - \tens{Z}\|_F.
\end{split}
\] 

\subsection{Compressed Convolutional Layers}
\label{sec:method_main}

Recall that even if $\dim=2$, finding the SVD of $\Tmat_{\tens{K}}$ requires $\mathcal{O}(n^2 (\ch^2 + \log n))$ FLOPs, which is too much for practical use with large networks.
To reduce the computational cost of SVD of $\Tmat_{\tens{K}}$, we propose a low-rank compressed layer representation based on the following tensor decomposition:
\begin{equation}\label{eq:tt}
    \tens{K}_{p_1 \dots p_\dim ij} = \sum_{\alpha = 0}^{r_1-1}\, \sum_{\beta = 0}^{r_2-1}   
    \tens{K}^{(1)}_{i\alpha}\ \tens{K}^{(2)}_{ p_1\dots p_\dim\alpha \beta}\ \tens{K}^{(3)}_{\beta j},
\end{equation}
where
$\tens{K}^{(1)}\in\mathbb{R}^{\cin\times r_1}$, $\tens{K}^{(3)}\in\mathbb{R}^{r_2\times \cout}$ and
$\tens{K}^{(2)}\in\mathbb{R}^{k \times \dots \times k \times r_1\times r_2}$ are some small tensors and integers $r_1,r_2$: $1\leq r_1 \leq \cin$, $1 \leq r_2 \leq \cout$ are called ranks.
This decomposition is essentially the TT decomposition~\citep{oseledets2011tensor} of the kernel tensor with filter modes $p_1,\dots,p_\dim$ merged into a single index.
We also note that it is equivalent to the so-called Tucker-2 decomposition, successfully used by~\citet{tucker} to compress large convolutional networks. The proposed decomposition is visualized in Fig.~\ref{fig:penrose}.
We note that we chose the TT decomposition as it gives us access to the SVD decomposition of the convolution (Theorem~\ref{thm:main}). 
It is, however, unclear how to obtain similar results with other decompositions, such as tensor-ring or canonical tensor decomposition, as they do not admit SVD-like form, see, e.g.,~\citep{grasedyck2013literature} for more details of these formats.

By substituting~\eqref{eq:tt} into~\eqref{eq:assumption}, using bilinearity of $'*'$ and omitting summation limits for clarity:
\[
\begin{split}
   &\left(\convop_{\tens{K}}(\tens{X})\right)_{j,:,\dots,:} =   \sum_i\tens{K}_{:,\dots,:,i,j} * \tens{X}_{i,:,\dots,:}
   =
     \sum_{i} \left(\sum_{\alpha,\beta}
    \tens{K}^{(1)}_{i\alpha}\ \tens{K}^{(2)}_{ :,\dots, :,\alpha, \beta}\ \tens{K}^{(3)}_{\beta j}\right)
    * \tens{X}_{i,:,\dots,:} = \\
    &
    \sum_{i}\sum_{\alpha,\beta}
    \tens{K}^{(1)}_{i\alpha}  \tens{K}^{(3)}_{\beta j} \left( \tens{K}^{(2)}_{ :,\dots, :,\alpha, \beta}
    * \tens{X}_{i,:,\dots,:}\right) =
     \sum_{\beta} 
     \tens{K}^{(3)}_{\beta j} \sum_\alpha  \tens{K}^{(2)}_{ :,\dots, :,\alpha, \beta}
    * \left( \sum_{i} \tens{K}^{(1)}_{i\alpha} \tens{X}_{i,:,\dots,:} \right),
\end{split}
\]
we obtain that the convolution with kernel given by~\eqref{eq:tt} is equivalent to a sequence of three convolutions without nonlinear activation functions between them:
\begin{enumerate}
    \item $1\times \dots \times 1$ convolution with $\cin$ input and $r_1$ output channels (kernel tensor
    -- $\tens{K}^{(1)}$);
    \item $k\times\dots \times k$ convolution with $r_1$ input and $r_2$ output channels (kernel tensor -- $\tens{K}^{(2)}$);
    \item $1\times\dots \times 1$ convolution with $r_2$ input and $\cout$ output channels (kernel tensor -- $\tens{K}^{(3)}$);
\end{enumerate}
or equivalently,
\begin{equation} 
\label{eq:comp}
    \convop_{\tens{K}} = \convopone_{\tens{K}^{(3)}} \circ \convop_{\tens{K}^{(2)}} \circ \convopone_{\tens{K}^{(1)}},
\end{equation}
where $\mathscr{A}\circ \mathscr{B}$ denotes the composition of two maps $\mathscr{A}$ and $\mathscr{B}$; $\convopone_{\tens{K}^{(i)}}$ denotes $1\times\dots\times 1$ convolution map given by a convolution kernel $\tens{K}^{(i)}$.

In the rest of this section, we show that singular values of the original convolutional layer with kernel $\tens{K}$ correspond to singular values of the compressed kernel $\tens{K}^{(2)}$.
As a first step, we demonstrate that we can impose orthogonality constraints on $\tens{K}^{(1)}$ and $\tens{K}^{(3)}$ without loss of expressivity.

\paragraph{Lemma \arabic{lemma}.}
\refstepcounter{lemma}\label{lemma:ortho}
\textit{
Let $\tens{K}\in\mathbb{R}^{k\times\dots\times k \times \cin \times \cout}$ be given by~\eqref{eq:tt}. Then there exist matrices ${\tens{Q}}^{(1)}\in\mathbb{R}^{\cin\times r_1}$, ${\tens{Q}}^{(3)}\in\mathbb{R}^{r_2\times \cout}$ satisfying 
  ${{\tens{Q}}^{(1)\top}}\tens{Q}^{(1)} = I_{r_1}$, ${\tens{Q}^{(3)}} {\tens{Q}^{(3)\top}} = I_{r_2}$
and a tensor ${\tens{Q}}^{(2)}\in\mathbb{R}^{k \times\dots\times k \times r_1\times r_2}$, such that
\begin{equation}\label{eq:ttqr}
    \tens{K}_{p_1 \dots p_\dim ij} = \sum_{\alpha = 0}^{r_1-1}\, \sum_{\beta = 0}^{r_2-1}   
    \tens{Q}^{(1)}_{i\alpha}\ \tens{Q}^{(2)}_{ p_1 \dots p_\dim \alpha \beta}\ \tens{Q}^{(3)}_{\beta j}.
\end{equation}
}
\begin{proof}
See Sec.~\ref{sec:appB1} of the Appendix.
\end{proof}

Now we have all preliminaries in place to formulate the key result that our framework is based on.

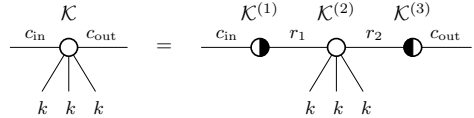
\begin{wrapfigure}{R}{0.45\textwidth}
  \vspace{-1.5em}
  \centering
  \resizebox{\linewidth}{!}{%
    \begin{tikzpicture}
    \newcommand{\rad}{0.15}
    \newcommand{\httsize}{1.3}
    \newcommand{\vttsize}{0.6}
    \pgfmathtruncatemacro{\step}{1}

	\node[label={[label distance=6pt]90:$\mathcal{K}$}] (p_left) at (-2.5*\httsize, 0) {};
	\path[-] (p_left) edge node[anchor=center, above] {\footnotesize $\cin$} ($(p_left) - (\httsize - 2 * \rad,0)$);
	\path[-] (p_left) edge node[anchor=center, above] {\footnotesize $\cout$} ($(p_left) + (\httsize - 2 * \rad,0)$);
	\path[-] (p_left) edge node[anchor=south, left=5pt, below=10pt] {\footnotesize $k$} ($(p_left) - (0.35 * \httsize, \step * \rad + \vttsize)$);
	\path[-] (p_left) edge node[anchor=south, right=1pt, below=10pt] {\footnotesize $k$} ($(p_left) - (0, \step * \rad + \vttsize)$);
	\path[-] (p_left) edge node[anchor=south, right=7pt, below=10pt] {\footnotesize $k$} ($(p_left) - (-0.35 * \httsize, \step * \rad + \vttsize)$);
	\draw[thick, fill=white] (p_left) circle (\rad);

	\node (p_equiv) at (-1.25*\httsize, 0) {$=$};

	\node[label={[label distance=6pt]90:$\mathcal{K}^{(1)}$}] (p_k1) at (0*\httsize, 0) {};
	\node[label={[label distance=6pt]90:$\mathcal{K}^{(2)}$}] (p_k2) at (1*\httsize, 0) {};
	\node[label={[label distance=6pt]90:$\mathcal{K}^{(3)}$}] (p_k3) at (2*\httsize, 0) {};
	\path[-] (p_k1) edge node[anchor=center, above] {\footnotesize $\cin$} ($(p_k1) - (\httsize - 2 * \rad,0)$);
	\path[-] (p_k1) edge node[anchor=center, above] {\footnotesize $r_1$} (p_k2);
	\path[-] (p_k2) edge node[anchor=center, above] {\footnotesize $r_2$} (p_k3);
	\path[-] (p_k3) edge node[anchor=center, above] {\footnotesize $\cout$} ($(p_k3) + (\httsize - 2 * \rad,0)$);

	\path[-] (p_k2) edge node[anchor=south, left=5pt, below=10pt] {\footnotesize $k$} ($(p_k2) - (0.35 * \httsize, \step * \rad + \vttsize)$);
	\path[-] (p_k2) edge node[anchor=south, right=1pt, below=10pt] {\footnotesize $k$} ($(p_k2) - (0, \step * \rad + \vttsize)$);
	\path[-] (p_k2) edge node[anchor=south, right=7pt, below=10pt] {\footnotesize $k$} ($(p_k2) - (-0.35 * \httsize, \step * \rad + \vttsize)$);

	\draw[thick, fill=white] (p_k1) circle (\rad);
    \draw[fill] ($(p_k1) + (-90:\rad)$) arc (-90:90:\rad) -- cycle ;

	\draw[thick, fill=white] (p_k2) circle (\rad);

	\draw[thick, fill=white] (p_k3) circle (\rad);
    \draw[fill] ($(p_k3) + (90:\rad)$) arc (90:270:\rad) -- cycle ;
\end{tikzpicture}
  }
  \caption{
    Tensor diagram representation of~\eqref{eq:tt} for $d=3$ with the constraints~\eqref{eq:framematr} as in Theorem~\ref{thm:main}. A node with $D$ ``legs'' represents a $D$-dimensional tensor; individual legs represent indices; connected legs represent summation over the respective indices; half-filled nodes are orthogonal matrices.
  }
  \label{fig:penrose}
  \vspace{-3em}
\end{wrapfigure}

\paragraph{Theorem \arabic{theorem}.}
\refstepcounter{theorem}\label{thm:main}
\textit{
Let $\tens{K}\in\mathbb{R}^{k\times\dots\times k \times \cin \times \cout}$ be given by~\eqref{eq:tt}. Let also ${\tens{K}}^{(1)}\in\mathbb{R}^{\cin\times r_1}$, ${\tens{K}}^{(3)}\in\mathbb{R}^{r_2\times \cout}$:
\begin{equation}\label{eq:framematr}
{{\tens{K}}^{(1)\top}} \tens{K}^{(1)} = I_{r_1}, \quad {\tens{K}^{(3)}} {\tens{K}^{(3)\top}} = I_{r_2}.
\end{equation}
Then the multiset of nonzero singular values of $\convop_{\tens{K}}$ defined in~\eqref{eq:assumption} equals the multiset of nonzero singular values of~$\convop_{\tens{K}^{(2)}}$.
}

\begin{proof}
See Sec.~\ref{sec:appB2} of the Appendix.
\end{proof}

Theorem~\ref{thm:main} suggests that after forcing the orthogonality of ${\tens{K}}^{(1)}$ and ${\tens{K}}^{(3)}$, one only needs to find singular values of the layer corresponding to the smaller kernel tensor ${\tens{K}}^{(2)}$ of the size $n\times n\times r_1 \times r_2$ instead of the original kernel tensor ${\tens{K}}$ of the size $n\times n \times r_1 \times r_2$.
At the same time, thanks to Lemma~\ref{lemma:ortho}, one may always set ${\tens{K}}^{(1)}$ and ${\tens{K}}^{(3)}$ to be orthogonal without additionally losing the expressivity of $\convop_{\tens{K}}$.
Fig.~\ref{fig:penrose} demonstrates the resulting decomposition using tensor diagram notation.

\section{Application of the Proposed Framework}
\label{sec:method}

This section discusses how to apply the proposed framework to various methods.
We start with the explicit formula for singular values computation and then discuss how to apply the framework when the convolutions are parametrized, for example, as in~\citep{singla2021skew,BCOP}.

\subsection{Explicit Formulas for Strided Convolutions}
\label{sec:statementtheorem2}

\citet{sedghi2019singular} consider periodic convolutions, allowing them to express singular values through discrete Fourier transforms and several SVDs.
We extend the main result of their work to the case of strided convolutions and summarize it in the following theorem.

\begin{theorem}
\label{thm:last_theorem}
Let $\widehat{\tens{K}} \in\mathbb{R}^{n\times n \times \cin \times \cout}$ be a kernel ${\tens{K}}\in\mathbb{R}^{k\times k \times \cin \times \cout}$ of a periodic convolution, padded with zeros along the filter modes.
Let $\stride$ be a stride of the convolution~$\convop_{\tens{K}}$: $n \equiv 0 \pmod \stride$ and
let the reshaped kernel $R \in \mathbb{R}^{\stride^2 \times \frac{n}{s} \times \frac{n}{s} \times \cin \times \cout}$ be such that:
\begin{equation} \label{eqn:reshape_for_fft}
    R_{q, a, b, i, j} = \widehat{\tens{K}}_{\lfloor q / s \rfloor + as, (q \hspace{-0.5em}\mod s) + bs, i, j}.
\end{equation}

Let us consider~{$\left( \frac{n}{\stride}\right)^2$} matrices $P^{(p_1, p_2)} \in \mathbb{R}^{\cin \times \stride^2 \cout}$ with entries: 
\begin{equation}
    \label{eqn:fft}
    P_{ij}^{(p_1, p_2)} = ( F^\top R_{j\hspace{-0.5em}\mod\stride^2, :, :, i, \lfloor j / \stride^2 \rfloor} F)_{p_1 p_2},
\end{equation}
where $F$ is an $n\times n$ Fourier matrix.
Then the multiset of singular values of $\Tmat_{{\tens{K}}}$ is as follows:
\begin{equation}
    \label{eqn:svd}
\sigma \left(\Tmat_{{\tens{K}}}\right) = \bigcup\limits_{p_1,p_2\in \{1,\dots,{\frac{n}{s}}\}}\sigma \left(P^{(p_1, p_2)}\right).
\end{equation}
\end{theorem}
\begin{proof} 
See Sec.~\ref{sec:prooftheorem2} of the Appendix.
\end{proof}

The algorithm for finding the singular values of a TT-compressed convolutional layer based on Theorem~\ref{thm:last_theorem} is summarized in Alg.~\ref{alg:clipping}.
This algorithm consists of several major parts. We start with a layer with the imposed TT structure as in~\eqref{eq:tt}. We reduce this decomposition to the form with orthogonality conditions~\eqref{eq:framematr} by using the QR decomposition (steps 1--3). The second part is the application of Theorem~\ref{thm:last_theorem} to the smaller kernel tensor $\widehat{\tens{K}}^{(2)}$ (steps 4--7). For detailed pseudocode of applying Theorem~\ref{thm:last_theorem}, see Sec.~\ref{sec:codeclip}.

\begin{algorithm}[t!]
  \caption{
    Singular values of a TT-compressed periodic convolutional layer ($d=2$).
  }%
  \label{alg:clipping}%
\begin{algorithmic}%
  \Require{ \\
    $\tens{K}^{(1)}\in \mathbb{R}^{\cin \times r_1}$ -- left core tensor of the TT-compressed kernel~$\tens{K}$, \\
    $\tens{K}^{(2)} \in \mathbb{R}^{k \times k \times r_1 \times r_2}$ -- middle core tensor of the TT-compressed kernel~$\tens{K}$, \\
    $\tens{K}^{(3)} \in \mathbb{R}^{r_2 \times \cin}$ -- right core tensor of the TT-compressed kernel~$\tens{K}$, \\
    $n$ -- input image size, \\
    $\stride$ -- stride of the convolution defined by kernel $\tens{K}$ s.t. $n \equiv 0 \pmod \stride$.
  }
  \Ensure{ \\
    Non-zero singular values of the convolutional layer, defined by the kernel tensor $\tens{K}$.
  }
\end{algorithmic}%
\begin{algorithmic}[1]%
  \Let{
    Q^{(1)}, R^{(1)}
  }{
    \texttt{QR}\left(\tens{K}^{(1)}\right)
  }\CommentGray{
    Perform QR decomposition of the left core%
  }
  \Let{
    Q^{(3)}, R^{(3)}
  }{
    \texttt{QR}\left(\tens{K}^{(3)\top}\right)
  }\CommentGray{
    Perform QR decomposition of the right core%
  }
  \Let{
  \widehat{\tens{K}}^{(2)}_{p_1, p_2, :, :}
  }{
    R^{(1)}\tens{K}^{(2)}_{p_1, p_2, :, :} R^{(3)\top},\ \forall p_1, p_2
  }\CommentGray{
    Absorb non-orthogonal factors in the middle core
  }
  \Let{
    \widehat{\tens{K}}^{(2)}
  }{
    \texttt{pad\_zeros}\left(\widehat{\tens{K}}^{(2)},\ \ \texttt{shape=}(n \times n \times r_1 \times r_2)\right)
  }\CommentGray{
    Pad middle core to the image size%
  }
  \Let{
    R
  }{
    \texttt{reshape}\left(\widehat{\tens{K}}^{(2)},\ \ \texttt{shape=}~\eqref{eqn:reshape_for_fft}\right)
  }\CommentGray{
    Reshape middle core as in~\eqref{eqn:reshape_for_fft}. See also Sec.~\ref{sec:codeclip}.
  }
  \Let{
    P_{ij}^{(:,:)}
  }{
    \texttt{fft2}\left(R_{j\hspace{-0.3em}\mod\stride^2, :, :, i, \lfloor j / \stride^2 \rfloor}\right),\ \ \forall\ i,j
  }\CommentGray{
    Perform $2$-dimensional FFT of slices of $R$%
  }
  \State\Return{
    $\bigcup\limits_{p_1,p_2\in \{1,\dots,\frac{n}{s}\}}\sigma \left(P^{(p_1, p_2)}\right)$
  }\CommentGray{
    Return a union of all singular values of $P^{(\cdot,\cdot)}$
  }
\end{algorithmic}
\end{algorithm}

Computing singular values using~\eqref{eqn:fft} and~\eqref{eqn:svd}
has the complexity~$\mathcal{O}(n^2 \ch^2(\ch s^2 + \log \frac ns))$, $\ch = \max\{\cin,\cout\}$, which depends cubically on $\ch$.
Thus, reducing $c$ can significantly reduce the computational cost of finding the singular values.
Let us now use the TT-representation~\eqref{eq:tt} for kernel tensor with the ranks $r_1 = r_2 = r$.
Theorem~\ref{thm:main} suggests that we only need to apply Theorem~\ref{thm:last_theorem} to the core $\tens{K}^{(2)}\in\mathbb{R}^{k\times k \times r_1 \times r_2}$, which leads to the complexity $\mathcal{O}(n^2 r^2(r + \log \frac ns))$. For example, $r = \ch/2$ provides a theoretical speedup of up to $8$ times (see Fig.~\ref{fig:example} for practical evaluation).
Memory consumption can also be high with the method of~\citet{sedghi2019singular}. 
Processing the padded kernel tensor of the size $n\times n \times \ch \times \ch$ with our method requires $4$ times less memory for $r=\ch/2$ and $9$ times less for $r=\ch/3$ (see also Fig.~\ref{fig:example}).

\begin{figure}[ht]
 \centering
    \subfigure{{\includegraphics[width=0.46\textwidth]{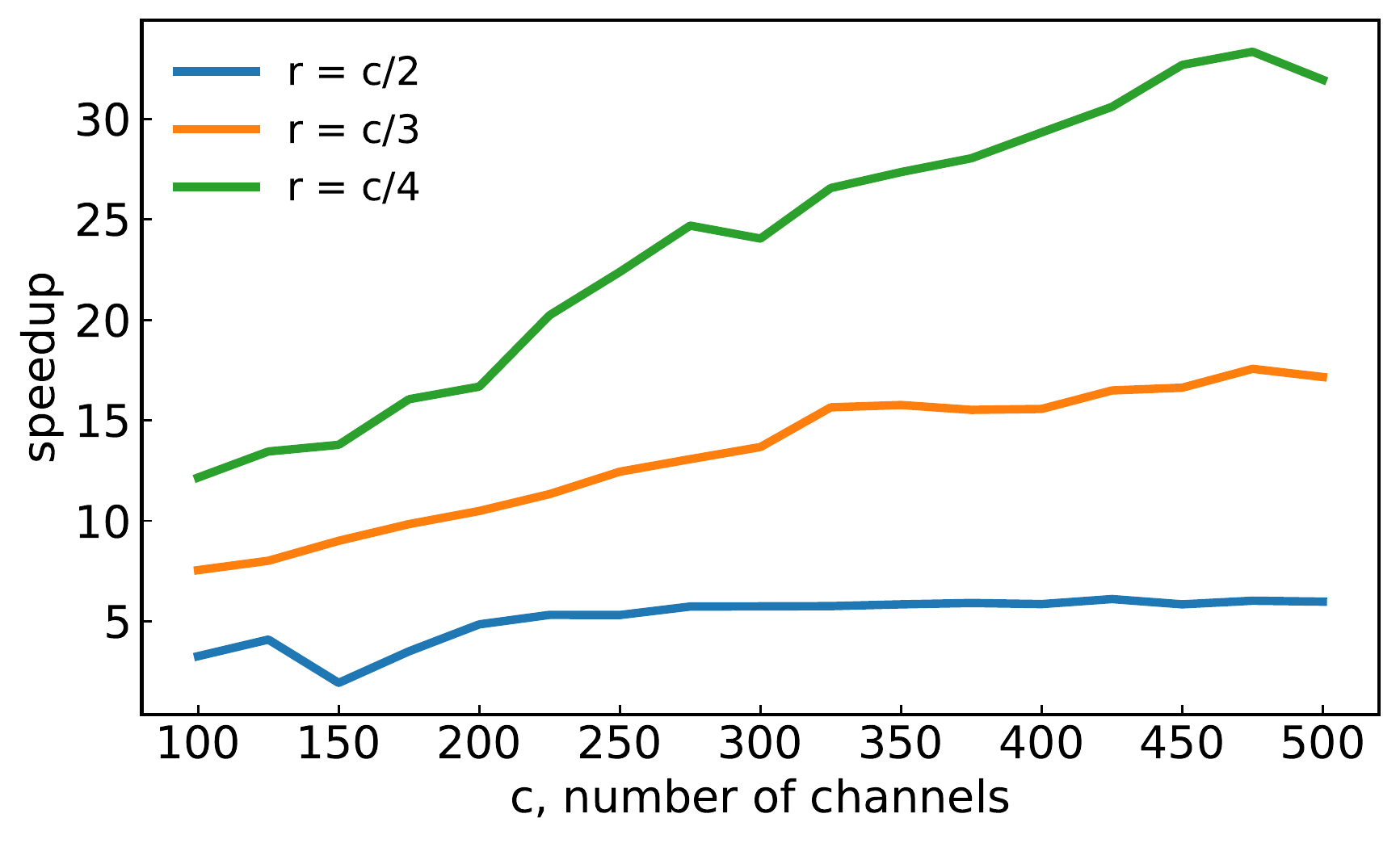} }}%
    \qquad
    \subfigure{{\includegraphics[width=0.46\textwidth]{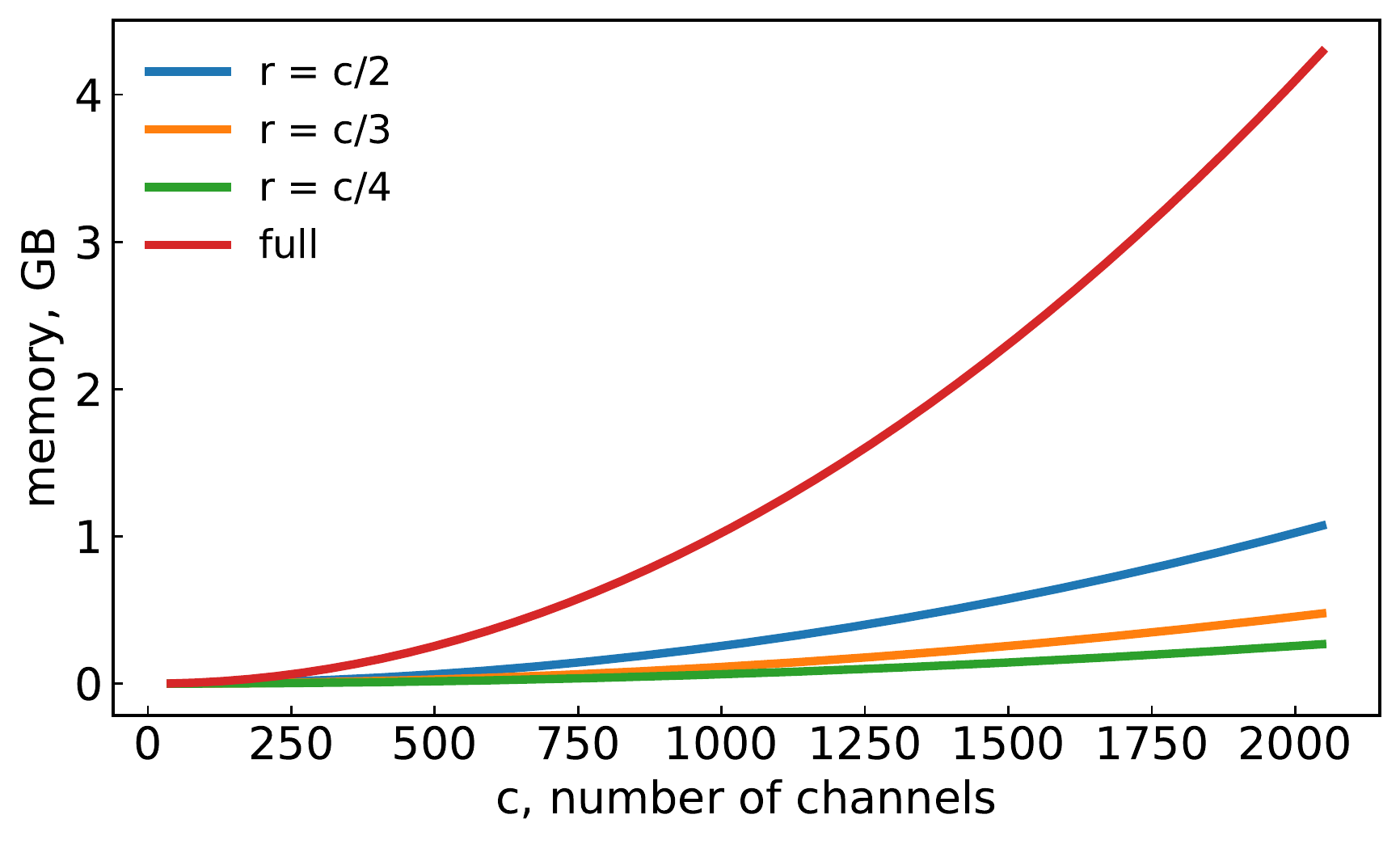} }}%
    \caption{
      Left: the speedup of computing singular values of a TT-compressed layer using Alg.~\ref{alg:clipping} relative to an uncompressed one. Right: Memory (single precision) to store the full $n\times n\times \ch \times \ch$ kernel and its TT-compressed form. Hyperparameter values: $n=16$ and $s=1$.
    }%
    \label{fig:example}%
\end{figure}

\subsection{Orthogonality Constraints} \label{sec:orthloss}

To ensure that the frame matrices $\tens{K}^{(1)}$ and $\tens{K}^{(3)}$ are orthogonal, we apply QR decomposition to the kernel in Alg.~\ref{alg:clipping}.
However, we can use other methods to maintain the orthogonality of frame matrices; for example, we can impose regularization on the layers.
Regularization is quite helpful in case $\tens{K}^{(2)}$ is not a standard convolution but has a special complex structure, such as the SOC layer from~\citep{singla2021skew}.
For that, let 
$n_l$ denote the number of TT-decomposed layers in a network, 
$\tens{K}^{(1)}_i$ and $\tens{K}^{(3)}_i$ -- left and right TT-cores of the $i^\mathrm{th}$ layer, and 
$r_{1,i}, r_{2,i}$ -- layer ranks.
Then, the orthogonal loss reads as:
\begin{equation}
\label{eq:orthloss}
    \text{loss}_\mathrm{ort} = 
    \left( 
    \sum\limits_{i=0}^{n_{l}} 
    \left\| \tens{K}_i^{(1)\top}  \tens{K}_i^{(1)} - I_{r_{1,i}} \right\|_F^2 + 
    \left\| \tens{K}_i^{(3)}  \tens{K}_i^{(3)\top} - I_{r_{2,i}} \right\|_F^2
    \right)
    \Bigg/
    \left(
    \sum\limits_{i=0}^{n_{l}} r_{1,i}^2 + r_{2,i}^2
    \right)
    .
\end{equation}

In the case of applying this regularization, we sum this loss with cross-entropy:
\begin{equation}\label{eq:lossortparam}
    \text{loss} = \text{loss}_\mathrm{CE} + \lambda_\mathrm{ort}\text{loss}_\mathrm{ort}.
\end{equation}
The value of $\text{loss}_\mathrm{ort}$ is the MSE measure of non-orthogonality of all $\tens{K}^{(1)}_i$ and $\tens{K}^{(3)}_i$, so the coefficient $\lambda_\mathrm{ort}$ tends to be large.
We set it to $1\text{e}5$ for experiments with WideResNets.

\section{Experiments}
\label{sec:experiments}

To study the effects of the proposed regularization on the performance of CNNs, we combined it with various methods of controlling singular values and applied it during training on the image classification task on the CIFAR-10 and CIFAR-100 datasets~\citep{krizhevsky2009learning}.
The code was implemented in PyTorch~\citep{pytorch} and executed on a NVIDIA Tesla V100, 32GB.
In our experiments, we used the LipConvNet architecture~\citep{singla2021skew}, a  WideResNet~\citep{Zagoruyko2016WideRN}, and VGG-19~\citep{vgg}.
Furthermore, the ranks are tested in the range $[c/4,c/2]$, where $c$ is the largest number of channels among all convolutions (except for $1\times 1$), which we found to give the better performance/speedup trade-off.

\paragraph{Evaluation Metrics}
We consider four main metrics: 
standard accuracy on the test split, 
accuracy on the CIFAR-C dataset (CC)~\citep{hendrycks2019robustness}\footnote{Distributed under Apache License 2.0}, 
accuracy after applying the AutoAttack module (AA)~\citep{autoattack}, 
and the Expected Calibration Error (ECE)~\citep{Guo2017OnCO}.
The last columns of tables with the results of our experiments contain compression ratios of all layers in the network.
In addition, we reported a dedicated metric for the LipConvNet architecture: p.r. -- provable bound-norm robustness, introduced by~\citet{BCOP}.
This metric shows the percentage of input images guaranteed to be predicted correctly despite any perturbation within radius $36/255$.

\subsection{Experiments on LipConvNet}
\label{sec:lipconvmain}
\begin{table*}[t!]
\caption{
Comparison between SOC~\citep{singla2021skew} and SOC combined with the proposed framework using LipConvNet-$N$ on CIFAR-10.
``Speedup'' is the speedup of training w.r.t. the SOC baseline.
``Comp.'' (compression) is the ratio between the number of parameters of convolutional layers in the original and the decomposed networks.
Despite speedups not being as substantial as in other experiments, we observe consistent improvement in all metrics, which is discussed in Fig.~\ref{fig:lip_ol_est}.
}
\label{tab:lc_sott_soc}
\vskip 0.15in
\begin{center}
\begin{small}
\begin{sc}
\begin{tabularx}{1.0\textwidth}{cc|ccccccc}
\toprule
rank & $N$ & Acc.~$\uparrow$ & AA~$\uparrow$ & CC~$\uparrow$ & ECE~$\downarrow$& p.r.~$\uparrow$ & speedup$\uparrow$& comp.$\uparrow$ \\
\midrule
-- & 5 & 75.6 ± 0.3 & 31.1 ± 0.2 & 67.3 ± 0.1 & 8.6 ± 0.4 & 59.1 ± 0.1 & 1.0 & 1.0 \\
 128 & 5 & 76.9 ± 0.2 & 32.8 ± 0.5 & 68.8 ± 0.2 & 6.7 ± 0.1 & 62.7 ± 0.1 & 1.4 & 2.9\\
 256 &\cellcolor{white}5 &\cellcolor{white}\textbf{78.3 ± 0.1} &\cellcolor{white}\textbf{34.7 ± 0.6} &\cellcolor{white}\textbf{69.7 ± 0.2} &\cellcolor{white}\textbf{5.3 ± 0.2} &\cellcolor{white}\textbf{65.4 ± 0.3} & 0.9 &\cellcolor{white}1.2\\
\midrule
-- & 20 & 76.3 ± 0.5 & 33.4 ± 0.5 & 68.1 ± 0.2 & 7.5 ± 0.1 & 61.4 ± 0.2 & 1.0 & 1.0\\
 128 & 20 & 76.8 ± 0.2 & 33.0 ± 0.1 & 68.3 ± 0.1 & 5.9 ± 0.2 & 62.4 ± 0.2 & 1.2 & 3.3 \\
 256 &\cellcolor{white}20 &\cellcolor{white}\textbf{78.4 ± 0.2} &\cellcolor{white}\textbf{35.4 ± 0.5} &\cellcolor{white}\textbf{70.4 ± 0.1} &\cellcolor{white}\textbf{4.7 ± 0.4} &\cellcolor{white}\textbf{65.6 ± 0.2} & 1.1 &\cellcolor{white}1.4 \\
\midrule
-- & 30 & 76.3 ± 1.0 & 32.0 ± 1.6 & 67.9 ± 1.0 & 7.3 ± 1.0 & 61.9 ± 1.2 & 1.0 & 1.0 \\
128 & 30 & 76.0 ± 0.1 & 32.6 ± 0.1 & 68.0 ± 0.1 & 5.0 ± 0.6 & 61.7 ± 0.6 & 1.3 & 3.5\\
\cellcolor{white}256 &\cellcolor{white}30 &\cellcolor{white}\textbf{77.9 ± 0.1}&\cellcolor{white}\textbf{34.7 ± 0.3}&\cellcolor{white}\textbf{69.7 ± 0.3}&\cellcolor{white}\textbf{4.8 ± 0.1}&\cellcolor{white}\textbf{64.8 ± 0.2}& 1.1 & \cellcolor{white}1.5\\
\bottomrule
\end{tabularx}
\end{sc}
\end{small}
\end{center}
\vskip -0.1in
\end{table*}

\citet{singla2021skew} introduced a new architecture called Lipschitz Convolutional Networks, or LipConvNet for short.
LipConvNet-$N$ consists of $N$ convolutional layers and activations.
This architecture is provably $1$-Lipschitz by contrast to popular residual network architectures.
In~\citet{singla2021skew}, the authors also propose an  orthogonal layer called Skew Orthogonal Convolution (SOC) and applied it in LipConvNet.

\begin{wrapfigure}{R}{0.45\textwidth}
  \centering\includegraphics[width=0.4\textwidth]{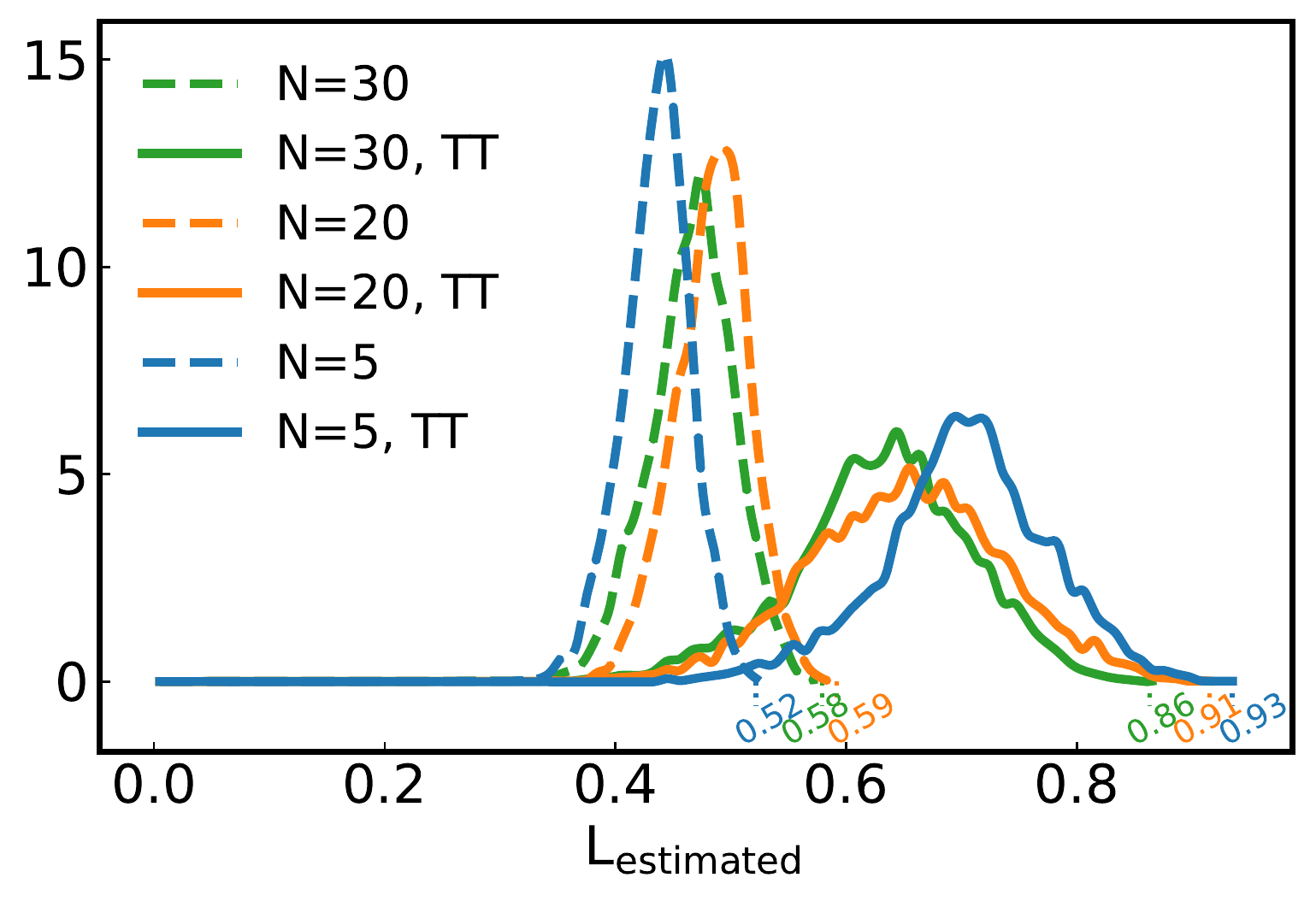}
  \caption{Histograms of empirically estimated Lipschitz constants of LipConvNet-$N$ architectures (see Sec.~\ref{sec:emlip}). When applying the TT method to LipConvNet architecture, the peaks are wider and closer to $1$, giving more flexibility to the model but still keeping constants below 1. Rank parameter: $r=256$.}
  \label{fig:lip_ol_est}
\end{wrapfigure}

We modify the SOC layer in correspondence with the proposed framework by adding a $1\times 1\times \cin \times r$ convolution before applying it, and finally applying the last convolution of size $1\times 1\times r\times \cout$, thus reducing the number of channels to $r\times r$ in the bottleneck.
To maintain the orthogonality of this new layer that we call SOC-TT, we add orthogonal loss (Sec.~\ref{sec:orthloss}) to keep these two frame matrices orthogonal.

Tab.~\ref{tab:lc_sott_soc} demonstrates metrics improvement after replacing SOC with SOC-TT.
For rank 256, we performed a grid search of the $\lambda_\mathrm{ort}$ coefficient.
The training protocol is the same as in~\citet{singla2021skew} except for the learning rate of LipConvNet-20 with the SOC-TT layer, which we set to 0.05 instead of~0.1.
To investigate the observed increase in metrics, we present histograms of empirically estimated Lischitz constants by using adversarial attacks in Fig.~\ref{fig:lip_ol_est} (see Sec.~\ref{sec:emlip}) with optimally chosen  $\lambda_\mathrm{ort}$ from \eqref{eq:lossortparam}. We observed that increasing $\lambda_\mathrm{ort}$ forces the frame matrices in the TT decomposition to be closer to orthogonal and, as a result, the histogram converges to that of the original SOC method.
At the same time, decreasing $\lambda_\mathrm{ort}$ may lead to networks with constants larger than $1$ but with better metrics.

\subsection{Experiments with WideResNet-16-10}
\label{sec:wrn2810}

Next, we consider WRN-16-10~\citep{Zagoruyko2016WideRN} with almost 17M parameters.
The results of our experiments on CIFAR-10 are summarised in Tab.~\ref{tab:wrn16}.
Additional experiments on CIFAR-100 are presented in Sec.~\ref{sec:appC_cifar100}.
The experimental setup and training schedule are taken from~\citet{gouk2020regularisation}.
We train our models with batches of 128 elements for 200 epochs, using SGD optimizer and a learning rate of 0.1, multiplied by 0.1 after 60, 120, and 160 epochs.
The hyperparameters for the experiments with division were taken from~\citet{gouk2020regularisation}.
The orthogonal loss~\eqref{eq:orthloss} was used in all experiments with decomposed networks.

\begin{table*}[t!]
\caption{
Performance metrics for different constraints applied to WRN-16-10 trained on CIFAR-10.
Both clipping and division combined with TT layers lead to an increase in metrics (except for ECE), with $256$  being the best overall rank.
Clipping tends to be a better option in terms of metrics but is substantially slower than the division approach.
Moreover, clipping the singular values to~2 results in better performance than clipping to~1.
``Speedup'' is a speedup of the overhead
resulting from singular values control
in all the layers in the network.
Clipping the whole network w/o decomposition takes 6.2 min while applying division -- 0.6 sec.
``Comp.'' (compression) is the ratio between the number of parameters of convolutional layers in the original ($\approx16.8$M) and decomposed networks.
}
\label{tab:wrn16}
\vskip 0.15in
\begin{center}
\begin{small}
\begin{sc}
\makebox[\textwidth][c]{
\begin{tabularx}{1.\textwidth}{cc|>{\centering\arraybackslash}Xc>{\centering\arraybackslash}X>{\centering\arraybackslash}Xcc}
\toprule
Method & Rank & Acc.~$\uparrow$ & AA~$\uparrow$ & CC~$\uparrow$ & ECE~$\downarrow$ & \multirow{1}{*}{Speedup~$\uparrow$} & \multirow{1}{*}{Comp.~$\uparrow$}\\ 
\midrule
\multirow{4}{*}{Baseline}& -- & \textbf{95.52} & 54.22 & 72.91 & \textbf{2.06} & -- & 1.0\\
& 192 & 95.21 & 54.97 & \textbf{73.94} & 2.84 & -- & 3.6\\
& 256 & 95.39 & \textbf{55.07} & 73.73 & 2.73 & -- & 2.4\\
& 320 & 95.02 & 52.32 & 72.86& 2.85 & -- & 1.8\\
\midrule
\multirow{4}{*}{Clip to 1} & -- & \textbf{95.27} & 53.27 & \textbf{73.44}& 2.45 & 1.0 & 1.0\\
& 192 & 95.25 & 51.7 & 73.37& 2.49 & 4.1 & 3.6\\
& 256 & 95.12 & 51.94 & 72.47 & \textbf{2.43} & 3.3 & 2.4\\
& 320 & 95.05 & \textbf{54.44} & 73.38 & 2.7 & 2.3 & 1.8\\
\midrule
\multirow{4}{*}{Clip to 2} & -- & \textbf{95.99} & 55.77 & 73.27 & \textbf{1.92} & 1.0 & 1.0\\
& 192 & 95.45 & 55.45 & 73.26 & 2.55 & 4.1 & 3.6\\
& 256 & 95.73 & \textbf{55.83} & 73.35 & 2.28 & 3.3 & 2.4\\
& 320 & 95.5 & 54.63 & \textbf{74.34} & 2.53 & 2.3 & 1.8\\
\midrule
\multirow{4}{*}{Division} & -- & 95.17 & 53.39 & 72.94 & 2.54 & 1.0 & 1.0 \\
& 192 & \textbf{95.71} & 53.4 & \textbf{74.68} & \textbf{2.4} & 2.3 & 3.6 \\
& 256 & 95.7 & \textbf{56.28} & 73.43 & 2.51 & 1.5 & 2.4 \\
& {320} & 94.93 & 52.35 & 70.89 & 2.97 & 1.3 & 1.8 \\
\bottomrule
\end{tabularx}
}
\end{sc}
\end{small}
\end{center}
\vskip -0.1in
\end{table*}

\paragraph{Clipping} Following~\citet{sedghi2019singular}, we apply the so-called clipping of singular values after computing the singular values of a layer using Theorem~\ref{thm:last_theorem}.
In particular, we first fix a threshold parameter $\delta$, chosen to be $1$ or $2$ in numerical experiments.
Then all singular values greater than $\delta$ are replaced with $\delta$.
This procedure allows for maintaining maximal singular values at the desired level.
After the clipping, we reconstruct the expanded kernel $\widehat{\tens{K}}_\delta$ of the size $n\times n \times \cin \times \cout$, which is no longer sparse.
Finally, we revert to its original shape by setting $\tens{K} = \widehat{\tens{K}}_\delta(1{:}k,1{:}k,:,:)$.
This operation is performed every 100 training iterations; we provide the code for computing singular values in Sec.~\ref{sec:codeclip} of the Appendix.
The results suggest that applying TT decomposition not only compresses the network by decreasing the number of parameters of its convolutional layers, but also speeds up the clipping operation.
Network robustness increased when training decomposed layers with clipping, which is confirmed by the results of the AutoAttack module.

\paragraph{Division}
Even though we can bound the Lipschitz constants of all convolutional layers, the Lipschitz constant of the whole network can be arbitrary.
In particular, for a WideResNet architecture, the Lipschitz constant would still be unconstrained due to residual connections, fully connected layers, batch normalization, and average pooling layers of the network.
\cite{gouk2020regularisation} derived a formula for the Lipschitz constants of batch normalization layers and proposed regularizing them too.
After each training step, they set the desired values for Lipschitz constants of convolutional, fully connected, and batch normalization layers.
To control the Lipschitz constant of convolutional and fully connected layers, they first compute the estimate of the largest singular value of a layer via power iteration.
Then they divide the layer by the estimate and multiply it by the desired Lipschitz constant.
A similar operation is performed for batch normalization layers, but the Lipschitz constant is computed directly.

In the original paper, this approach proved to be successful for WideResNet-16-10.
The method itself already increases the robust metrics and accuracy substantially.
The pattern remains intact after applying our decomposition to the layers: the robust accuracy of the AutoAttack module improves.
Overall, the TT decomposition, applied without any other regularization, performs poorly.
However, when used to facilitate the application of a particular method of control over the singular values, the decomposition only slightly decreases the accuracy, which might still exceed that of the baseline.

\subsection{Experiments with VGG-19}

\begin{figure}[t]
	\centering
	\includegraphics[width=1\textwidth]{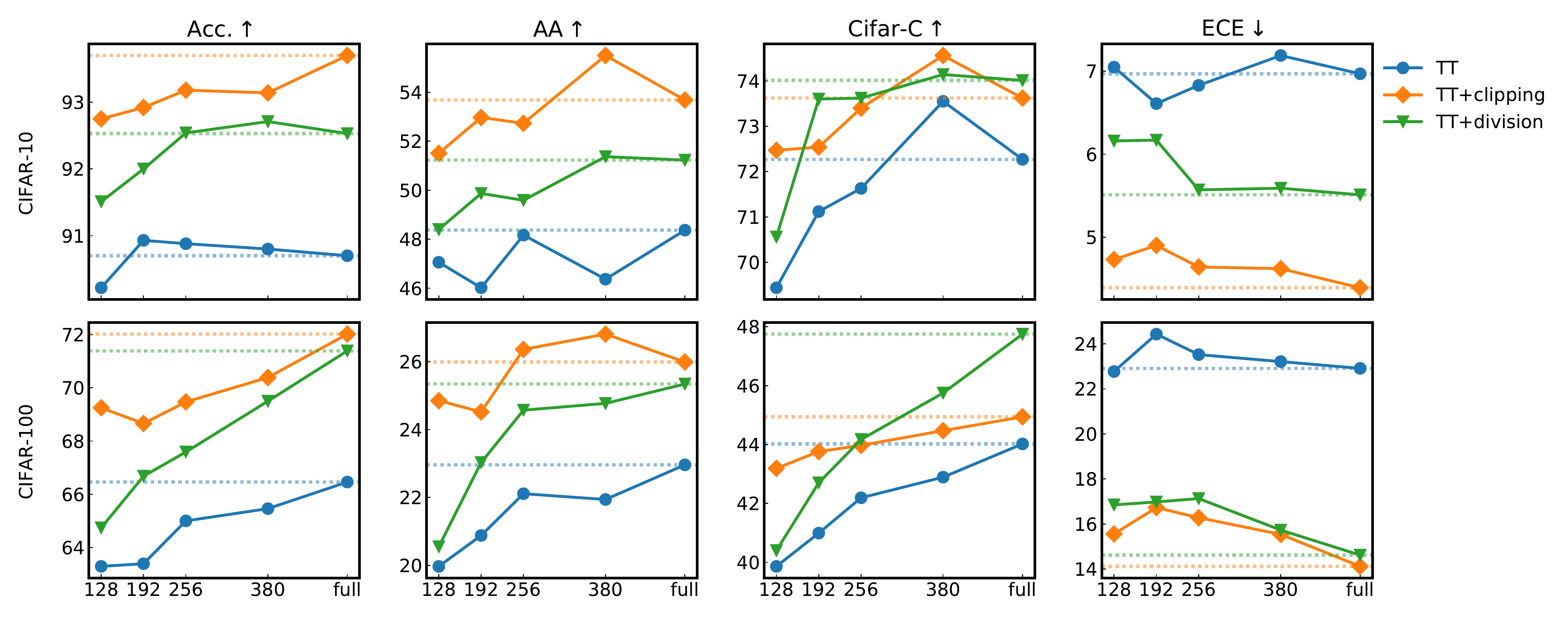}
	\vspace{-1.5em}
	\caption{Metrics for different ranks on CIFAR-10 and CIFAR-100; 
	``full'' stands for the uncompressed case, i.e., ``full'' label with TT method is simply the baseline with no TT decomposition and no control of singular values,
	similarly ``full'' TT+clipping is~\citep{sedghi2019singular} and 
	``full'' TT+division is~\citep{gouk2020regularisation}.
	Horizontal dotted lines additionally represent uncompressed baselines.
	}
	\label{fig:vgg_metrics}
\end{figure}
This section presents the results of controlling singular values on the VGG-19 network and two datasets: CIFAR-10 and CIFAR-100. As previously, we consider the clipping method with the exact computation of singular values~\citep{sedghi2019singular} and the application of one iteration of the power method~\citep{gouk2020regularisation}. The clipping parameter is set to $0.5$.
The results are presented in Fig.~\ref{fig:vgg_metrics}.
Firstly, we observe that all the considered robust metrics benefit from controlling singular values. 
This effect is  more pronounced than for the WRN-16-10 and was also observed in~\citep{gouk2020regularisation}.
We conjecture that this happens due to the lack of residual connections in VGG-type architectures, which makes them more sensitive to singular values of convolutional layers; in residual networks, convolutional layers are only corrections to the signal.
In almost all scenarios, we observed that the clipping method performed better than the division with the power method.

\section{Conclusion}
\label{sec:conclusion}
We proposed a principled sparsity-driven approach to controlling the singular values of convolutional layers.
We investigated two different families of CNNs and analyzed their performance and robustness under the proposed singular values constraint. 
The limitations of our approach are partly inherited from the encapsulated methods of computing singular values (e.g., the usage of periodic convolutions) and partly stem from the core assumption of CNN layer sparsity.
Overall, our approach has proven effective for large-scale networks with hundreds of input and output channels, making it stand out compared to intractable methods from the prior art. 
\paragraph{Societal Impact} As our work is concerned with improving the robustness of neural networks, it positively impacts their reliability. Nevertheless, due to the generality of our approach, it can be used in various application domains, including malicious purposes.

\section*{Acknowledgments}

The publication was supported by the grant for research centers in the field of AI provided by the Analytical Center for the Government of the Russian Federation (ACRF) in accordance with the agreement on the provision of subsidies (identifier of the agreement 000000D730321P5Q0002) and the agreement with HSE University \textnumero 70-2021-00139.
The calculations were performed in part through the computational resources of HPC facilities at HSE University~\citep{kostenetskiy2021hpc}.

{
\bibliography{paper}
\bibliographystyle{plainnat}
}

\newpage
\appendix

\title{
  Towards Practical Control of Singular Values of Convolutional Layers:\\
  Supplementary Materials
}
\maketitle
\vspace{-2em}
\blfootnote{$^*$Equal contribution.}
\blfootnote{$^\dagger$Corresponding author: Alexandra Senderovich (\href{mailto:AlexandraSenderovich@gmail.com}{AlexandraSenderovich@gmail.com})}
\section{Compressed Convolutional Layers}
\label{sec:appB}

\subsection{Proof of Lemma~\ref{lemma:ortho}}
\label{sec:appB1}

The result essentially follows from different ways to represent the TT decomposition~\citep{holtz2012manifolds}.
By applying the QR decomposition to $\tens{K}^{(1)}$ and ${\tens{K}^{(3)\top}}$, we obtain:
\[
    \tens{K}^{(1)} = \tens{Q}^{(1)} R^{(1)}, \quad {\tens{K}^{(3)\top}} = {\tens{Q}^{(3)\top}} R^{(3)},
\]
where $R^{(1)}$, $R^{(3)}$ are upper triangular.
Substituting these formulas into~\eqref{eq:tt} yields~\eqref{eq:ttqr} with
\[
    \tens{Q}^{(2)}_{p_1\dots p_\dim \alpha \beta} =
    \sum_{\alpha' = 0}^{r_1-1}\, \sum_{\beta' = 0}^{r_2-1}   
    R^{(1)}_{\alpha\alpha'}\tens{K}^{(2)}_{ p_1\dots p_\dim \alpha' \beta'} R^{(3)}_{\beta'\beta},
\]
which completes the proof.

\subsection{Proof of Theorem~\ref{thm:main}}
\label{sec:appB2}

Using~\eqref{eq:comp}, we have
\[
    \Tmat_{\tens{K}} = \Tmat_{\tens{K}^{(3)}} \Tmat_{\tens{K}^{(2)}} \Tmat_{\tens{K}^{(1)}}.
\]
Let us first show that given ${{\tens{K}}^{(1)\top}} \tens{K}^{(1)} = I_{r_1}$, the matrix $\Tmat_{\tens{K}^{(1)}}$ has orthonormal rows, i.e., it satisfies 
\begin{equation} \label{eq:ortho_k1}
    \Tmat_{\tens{K}^{(1)}} \Tmat^\top_{\tens{K}^{(1)}} = I_{r_1 n^\dim}.
\end{equation}
To do so, let us find $\Tmat_{\tens{K}^{(1)}}$ in terms of $\tens{K}^{(1)}$.
For any $\tens{X}\in\mathbb{R}^{\cin\times n\times \dots \times n}$ and its row-major reshaping into a matrix $X\in\mathbb{R}^{\cin\times n^d}$, we have
\[
\begin{split}
    \Tmat_{\tens{K}^{(1)}} &\vecs(\tens{X}) \equiv  \vecs \left(\convop_{{\tens{K}}^{(1)}}(\tens{X})\right) \\
    &=  \vecs \left( \tens{K}^{(1)} X \right) 
     = \left(\tens{K}^{(1)} \otimes I_{n^d}\right) \vecs \left( X \right)\\
     & = \left(\tens{K}^{(1)} \otimes I_{n^d}\right) \vecs \left( \tens{X} \right).
\end{split}
\]
Therefore, $\Tmat_{\tens{K}^{(1)}} = \tens{K}^{(1)} \otimes I_{n^d}$, where $\otimes$ denotes the Kronecker product of matrices.
Using basic Kronecker product properties and the orthogonality of $\tens{K}^{(1)}$, we arrive at~\eqref{eq:ortho_k1}.
Analogously, we may obtain $\Tmat_{\tens{K}^{(3)}} = {\tens{K}^{(3)\top}} \otimes I_{f(n,k)^d}$ and
\[
    \Tmat^\top_{\tens{K}^{(3)}} \Tmat_{\tens{K}^{(3)}} = I_{\cout f(n,k)^\dim}.
\]
Finally, using SVD of $\Tmat_{\tens{K}^{(2)}}$: $\Tmat_{\tens{K}^{(2)}} = U\Sigma V^\top$, we get:
\[
\begin{split}
    \Tmat_{\tens{K}} &= 
    \left({\tens{K}^{(3)\top}} \otimes I_{f(n,k)^d}\right)
    U\Sigma V^\top
    \left({\tens{K}^{(1)\top}} \otimes I_{n^d}\right) \\
    &= 
    \left(\left({\tens{K}^{(3)\top}} \otimes I_{f(n,k)^d}\right) U\right)
    \Sigma
    \left(\left({\tens{K}^{(1)}} \otimes I_{n^d}\right)V\right)^\top = 
    \widetilde U \Sigma \widetilde{V}^\top,
\end{split}
\]
where $\widetilde U, \widetilde V$ have orthonormal columns as a product of matrices with orthonormal columns.
Hence, $\Tmat_{\tens{K}} = \widetilde U \Sigma \widetilde{V}^\top$ is in the compact SVD form, which completes the proof.

\section{Periodic Strided Convolutions} 
\label{sec:appA}

\subsection{Code for Computing Singular Values of a Strided Convolution}
\label{sec:codeclip}

According to Theorem~\ref{thm:last_theorem}, the following code computes the singular values of a transform encoded by a strided convolution.
Note that neither the full SVD nor the clipping operation is included in the code.
The code for computing the singular vectors and the new kernel with constrained singular values can be found in the source code repository.

\begin{verbatim}
def SingularValues(kernel, input_shape, stride):
 1  kernel_tr = np.transpose(kernel, axes=[2, 3, 0, 1])
 2  d1 = input_shape[0] - kernel_tr.shape[2]
 3  d2 = input_shape[1] - kernel_tr.shape[3]
 4  kernel_pad = np.pad(kernel_tr, ((0, 0), (0, 0), (0, d1), (0, d2)))
 5  str_shape = input_shape // stride
 6  r1, r2 = kernel_pad.shape[:2]
 7  transforms = np.zeros((r1, r2, stride**2, str_shape[0], str_shape[1]))
 8  for i in range(stride):
 9      for j in range(stride):
10          transforms[:, :, i*stride+j, :, :] = \
11                  kernel_pad[:, :, i::stride, j::stride]
12      transforms = np.fft.fft2(transforms)
13      transforms = transforms.reshape(r1, -1, str_shape[0], str_shape[1])
14      transpose_for_svd = np.transpose(transforms, axes=[2, 3, 0, 1])
15      sing_vals = svd(transpose_for_svd, compute_uv=False).flatten()
16      return sing_vals
\end{verbatim}

\subsection{Proof of Theorem 2}
\label{sec:prooftheorem2}

In this section, we prove Theorem~\ref{thm:last_theorem} from Sec.~\ref{sec:statementtheorem2}.
Firstly, we analyze the structure of the matrix corresponding to a strided convolution.
Secondly, we show that the columns of this matrix can be permuted to make matrix structure similar to that of a  non-strided convolution.
Therefore, the new theorem can be reduced to the already proven theorem.

Let us denote a circulant with each row shifted by the value of stride as a ``strided circulant''.
The shape of such a strided circulant is $\frac{n}{\stride} \times n$, where $n$ is the number of elements in the first row.
Here is an example of a strided circulant with $n = 4$, $\stride=2$:
\[
\begin{pmatrix}
a & b & c & d\\
c & d & a & b
\end{pmatrix}.
\]
One can think of this strided circulant as a block-circulant matrix with block sizes $1 \times \stride$.
At the same time, slicing this matrix by taking columns with a step equal to the stride, e.g., columns $\left (0, \stride, 2 \cdot \stride, \ldots \left( \frac{n}{\stride} - 1 \right) \cdot \stride \right)$, gives a standard circulant matrix.
This fact can be stated as
\[
A = \sum\limits_{i = 0}^{\stride - 1} A_i \left(I_{\frac{n}{s}}\otimes e_i^T \right),
\]
where $\otimes$ denotes the Kronecker product, $A$ is a strided circulant, $e_i \in \mathbb{R}^{\stride}$ is the $i$-th standard basis vector in $\mathbb{R}^{\stride}$, and $A_i$ is a circulant obtained by slicing columns of~$A$.

The regular 2D convolutional transform matrix with a single input and single output channels has a doubly block-circulant structure (see Section 5.5 in \citet{Jain}).
However, for strided convolutions, the structure is different.
For a fixed pair of input and output channels, the output of a strided convolution is a submatrix of an output of a convolution with the same kernel but without the stride.
More specifically, it is a slice with the stride $\stride$ by both dimensions (in Python, it would be written as \texttt{[::$\stride$, ::$\stride$]}).
For the matrix encoding the transformation, it means that only every $\stride$-th block row (simulating the slice by the first dimension) and every $\stride$-th regular row of a block (simulating the slice by the second dimension) are considered.
This means that the doubly block-circulant structure of the initial matrix turns into a doubly block-strided circulant structure of shape $\left(\frac ns\right)^2 \times n^2$, where $n$ is the size of a kernel.

Each block of this matrix is a strided circulant, and the block structure is that of a strided circulant as well.
The matrix consists of $\frac{n}{\stride} \times n$ blocks, and each of them has the shape $\frac{n}{\stride} \times n$.
If we denote a strided circulant of a row vector $a$ as $\mathrm{circ}_{\stride}(a),$ then the matrix of the transform is as follows:
\[
B = \begin{pmatrix}
\mathrm{circ}_{\stride}({\tens{K}}_{0,:})&\dots&\mathrm{circ}_{\stride}({\tens{K}}_{n{\shortminus}1,:})\\
\mathrm{circ}_{\stride}({\tens{K}}_{n{\shortminus}s,:})&\dots&\mathrm{circ}_{\stride}({\tens{K}}_{n{\shortminus}s{\shortminus}1,:})\\
\vdots&\ddots&\vdots\\
\mathrm{circ}_{\stride}({\tens{K}}_{s,:})&\dots&\mathrm{circ}_{\stride}({\tens{K}}_{s{\shortminus}1,:})\\
\end{pmatrix}.
\]
It can be noted that, in the same way as in strided circulants, we can take block columns of the matrix with a step equal to the stride and get block circulants (where each block is a strided circulant).
This block structure can be described as follows:
\[
B = \sum\limits_{i = 0}^{\stride - 1} B_i \left( \left( I_{\frac{n}{s}}\otimes e_i^T \right) \otimes I_{n} \right).
\]
Here $B$ is a doubly block-strided circulant matrix, and $B_i \in \mathbb{R}^{(\frac{n}{s})^2 \times \frac{n^2}{s}}$ is a block-circulant matrix.
This formula is similar to our previous sum representation for a strided circulant; however, the dimensions of the right term of each element are larger to account for the block structure of the left term.
Note that this term is needed to describe how the columns of $B_i$ are positioned in the matrix $B$, similar to the 1D case.

Let us consider $B_i$.
It consists of block columns $\left( i, s + i, 2s + i, \ldots \left( \frac{n}{s} - 1 \right) s + i \right)$:
\[
B_i = \begin{pmatrix}
\mathrm{circ}_{\stride}({\tens{K}}_{i,:})&\dots&\mathrm{circ}_{\stride}({\tens{K}}_{n{\shortminus}s + i,:})\\
\mathrm{circ}_{\stride}({\tens{K}}_{n{\shortminus}s + i,:})&\dots&\mathrm{circ}_{\stride}({\tens{K}}_{n{\shortminus}2s+i,:})\\
\vdots&\ddots&\vdots\\
\mathrm{circ}_{\stride}({\tens{K}}_{s + i,:})&\dots&\mathrm{circ}_{\stride}({\tens{K}}_{i,:})\\
\end{pmatrix}
\]
As a block-circulant with $\frac{n}{s} \times \frac{n}{s}$ blocks of $\frac{n}{s} \times n$, it can be expanded as follows:
\[
B_i = \sum\limits_{k = 0}^{\frac{n}{s} - 1} P^{k} \otimes C_{ik},
\]
where $C_{ik} \in \mathbb{R}^{\frac{n}{s} \times n}$ is a strided circulant block, $C_{ik} = \mathrm{circ}_{\stride} \left( {\tens{K}}_{i - sk, :} \right)$ and $P$ is a permutation matrix:
\[
P = \begin{pmatrix}
  0 & \ldots & 0 & 1\\
  1 & \ldots & 0 & 0\\
  \vdots & \ddots & \vdots & \vdots\\
  0 & \ldots & 1 & 0
\end{pmatrix} \in \mathbb{R}^{\frac{n}{s} \times \frac{n}{s}}.
\]
We can expand $C_{ik}$ as a strided circulant:
\[
C_{ik} = \sum\limits_{j = 0}^{s - 1}A_{ikj} \left( I_{\frac{n}{s}} \otimes e_{j}^T \right),
\]
where $A_{ikj} \in \mathbb{R}^{\frac{n}{s} \times \frac{n}{s}}$ is a regular circulant matrix that can be acquired as $\mathrm{circ}_1({\tens{K}}_{i - sk, j::s}),$ i.e. as a circulant built from the slice of a string ${\tens{K}}_{i - sk},$ taken with a step $s$ starting from the index $j$.
Finally,
\[
\begin{split}
B_i &= \sum\limits_{k = 0}^{\frac{n}{s} - 1} P^{k} \otimes C_{ik} = \sum\limits_{k = 0}^{\frac{n}{s} - 1} P^{k} \otimes \left(\sum\limits_{j = 0}^{s - 1}A_{ikj} \left( I_{\frac{n}{s}} \otimes e_{j}^T \right) \right) = \sum\limits_{k = 0}^{\frac{n}{s} - 1} \sum\limits_{j = 0}^{s - 1} P^{k} \otimes \left( A_{ikj} \left( I_{\frac{n}{s}} \otimes e_{j}^T \right) \right)\\
&= \sum\limits_{k = 0}^{\frac{n}{s} - 1} \sum\limits_{j = 0}^{s - 1} P^{k} \otimes \left( A_{ikj} \otimes e_{j}^T \right) = \sum\limits_{k = 0}^{\frac{n}{s} - 1} \sum\limits_{j = 0}^{s - 1} \left( \left( P^{k} \otimes A_{ikj} \right) \otimes e_{j}^T \right) = \sum\limits_{j = 0}^{s - 1}\left(\sum\limits_{k =0}^{\frac{n}{s} - 1} P^{k} \otimes A_{ikj}\right)\otimes e_{j}^T\\
&= \sum\limits_{j = 0}^{s - 1}\left(\sum\limits_{k =0}^{\frac{n}{s} - 1} P^{k} \otimes A_{ikj}\right) \left( I_{(\frac{n}{s})^2}\otimes e_{j}^T \right).
\end{split}
\]
This reformulation helps us see that the slices of $B_i$ by columns with step $s$ are, in fact, doubly block-circulant matrices defined by spatial slices of the kernel.
There are $s$ matrices $B_i$, each containing $s$ column slices.
The $j$th column slice of $B_i$ is defined by  $A_{i,:,j}$, which, in turn, is defined by the kernel slice ${\tens{K}}_{i::s, j::s}.$

In order to reduce the task of computing singular values of the matrix $B$, corresponding to the strided convolution, to the simpler task of computing singular values of a matrix corresponding to a regular convolution, let us permute the columns of $B_i$:
\[
\begin{split}
B'_i = \sum\limits_{j = 0}^{s - 1} \left[ e_{j}^T \otimes \left(\sum\limits_{k =0}^{\frac{n}{s} - 1} P^{k} \otimes A_{ikj}\right)\right].
\end{split}
\]
This matrix consists of $\stride$ consecutive doubly block-circulant matrices $\operatorname{circ}^2\left({\tens{K}}_{i::s, j::s}\right)$.
The first dimension of $B_i$ is the same as the first dimension of $B$.
Therefore, for $B$, this is also just a permutation of the columns.
The next step is to permute the block columns of $B$:
\[
\begin{split}
B' = \sum\limits_{i = 0}^{s - 1} \left(e_i^T \otimes \sum\limits_{j = 0}^{s - 1} \left[ e_{j}^T \otimes \left(\sum\limits_{k =0}^{\frac{n}{s} - 1} P^{k} \otimes A_{ikj}\right)\right]\right).
\end{split}
\]
After this permutation, matrix $B$ consists of $\stride^2$ consecutive doubly block-circulant matrices.
Note that the particular order of these blocks is not important, as it is just a matter of column order.

Finally, let us look at the matrix associated with the multiple-channel convolution.
The matrix of the convolutional transform with $\cin$ input channels and $\cout$ output channels is as follows (equivalent to the structure described in~\citet{sedghi2019singular}):
\[M = \begin{pmatrix}
B_{0,0}&B_{0,1}&\dots&B_{0,(\cout{\shortminus}1)}\\
B_{1,0}&B_{1,1}&\dots&B_{1,(\cout{\shortminus}1)}\\
\vdots&\vdots&\ddots&\vdots\\
B_{(\cin{\shortminus}1),0}&B_{(\cin{\shortminus}1),1}&\dots&B_{(\cin{\shortminus}1),(\cin{\shortminus}1)}\\
\end{pmatrix},
\]
where
\[
B_{c,d} = \mathrm{circ}_{\stride}(\tens{K}_{:,:,c,d}).
\]
However, now we know that we can permute the columns of this matrix to get a matrix comprised of doubly block-circulant blocks.
After the permutation described above, we get a matrix
\[M' = \begin{pmatrix}
B'_{0,0}&B'_{0,1}&\dots&B'_{0,(\stride^2\cout{\shortminus}1)}\\
B'_{1,0}&B'_{1,1}&\dots&B'_{1,(\stride^2\cout{\shortminus}1)}\\
\vdots&\vdots&\ddots&\vdots\\
B'_{(\cin{\shortminus}1),0}&B'_{(\cin{\shortminus}1),1}&\dots&B'_{(\cin-1),(\stride^2\cout{\shortminus}1)}\\
\end{pmatrix},
\]
where
\[
B'_{c,d} = \operatorname{circ}^2\left({\tens{K}}_{i::s, j::s, c, \lfloor d / \stride^2\rfloor}\right),
\]
$i, j$ -- some integer indices from 0 to $\stride - 1$.
The exact relationship between $i, j$ and $c, d$ is not important, as it depends only on the order of the columns.
The only important thing is that it has to be the same in all the rows.
We choose the following functional form:
\[i = \left \lfloor \left(d\hspace{-0.5em}\mod\stride^2\right)/ s \right \rfloor,\ \ j = d\hspace{-0.5em}\mod\stride.
\]
This matrix $M'$ can be perceived as the matrix of convolution for a new kernel $K' \in \mathbb{R}^{\frac{n}{s} \times \frac{n}{s} \times \cin \times \cout\stride^2},$ defined by this equation:
\[
K_{a, b, c, d} = \widehat{\tens{K}}_{ \lfloor \left(d\hspace{-0.5em}\mod\stride^2\right)/ s\rfloor + a\stride, d\hspace{-0.5em}\mod\stride + b\stride, c, d\hspace{-0.5em}\mod\stride}.
\]
Alternatively, to make things simpler, we can use the additional tensor $R$, defined in~\eqref{eqn:reshape_for_fft}.
It is easier to use this tensor for implementing the formula in Python.

To conclude, we reduced the task of computing the singular values of $M$ to the task of computing the singular values of the $M'$, solved by~\citet{sedghi2019singular}.
The reduction is made possible via columns permutation, meaning that the singular values of the matrix did not change.

\section{Additional Empirical Studies}

\begin{table*}[t!]
\caption{
Various metrics for the proposed framework applied to the SOC method (SOC-TT) and the LipConvNet-$N$ architectures. $\lambda_\mathrm{ort}$ denotes the regularization parameter of the orthogonal loss~\eqref{eq:orthloss}. We chose a range of lambda values to allow us to keep the Lipschitz constant under $1$. 
The rank is set to $256$.
}
\label{tab:lc_ol}
\begin{center}
\begin{small}
\begin{sc}
\begin{tabular}{cc|ccccc}
\toprule
\multicolumn{2}{c|}{Parameters} & \multicolumn{5}{c}{Metrics}\\
-$N$ & $\lambda_\mathrm{ort}$ & Acc.~$\uparrow$ & CIFAR-C~$\uparrow$ & ECE~$\downarrow$ & AA~$\uparrow$ & Lip.~$\downarrow$ \\ 
\midrule
\multirow{6}{*}{5}
&5e3 &\textbf{78.37} &\textbf{69.84} &\textbf{5.42} &\textbf{34.65} &\textbf{0.93} \\
&8e3 &77.68 &69.57 &6.57 &33.73 &0.79 \\
&\cellcolor{white}1e4 &\cellcolor{white}77.17 &\cellcolor{white}68.99 &\cellcolor{white}6.61 &\cellcolor{white}32.13 &\cellcolor{white}0.76 \\
&3e4 &76.25 &67.9 &8.35 &30.96 &0.6 \\
&\cellcolor{white}4e4 &\cellcolor{white}75.8 &\cellcolor{white}67.67 &\cellcolor{white}7.96 &\cellcolor{white}30.69 &\cellcolor{white}0.59 \\
&5e4 &75.74 &67.66 &8.71 &30.77 &0.58 \\
\midrule
\multirow{6}{*}{20}&7e4 &\textbf{78.41} &\textbf{70.47} &\textbf{4.89} &\textbf{35.89} &\textbf{0.91} \\
&\cellcolor{white}8e4 &\cellcolor{white}78.34 &\cellcolor{white}70.27 &\cellcolor{white}5.18 &\cellcolor{white}35 &\cellcolor{white}0.89 \\
&1e5 &77.45 &69.81 &5.43 &34.1 &0.8 \\
&\cellcolor{white}2e5 &\cellcolor{white}76.56 &\cellcolor{white}68.4 &\cellcolor{white}6.62 &\cellcolor{white}32.19 &\cellcolor{white}0.64 \\
&3e5 &76.03 &68 &7.23 &32.52 &0.58 \\
&\cellcolor{white}4e5 &\cellcolor{white}75.66 &\cellcolor{white}67.56 &\cellcolor{white}7.08 &\cellcolor{white}31.54 &\cellcolor{white}0.57 \\
\midrule
\multirow{6}{*}{30}&2e5 &\textbf{77.82} &\textbf{69.94} &\textbf{4.78} &\textbf{34.91} &\textbf{0.86} \\
&3e5 &77.86 &69.38 &5.86 &33.96 &0.75 \\
&\cellcolor{white}5e5 &\cellcolor{white}76.57 &\cellcolor{white}68.23 &\cellcolor{white}6.79 &\cellcolor{white}32.73 &\cellcolor{white}0.65 \\
&7e5 &75.93 &68.03 &6.93 &32.65 &0.59 \\
&\cellcolor{white}9e5 &\cellcolor{white}75.88 &\cellcolor{white}67.71 &\cellcolor{white}6.81 &\cellcolor{white}31.49 &\cellcolor{white}0.59 \\
&1.2e6 &76.22 &67.6 &7.23 &31.85 &0.56 \\
\bottomrule
\end{tabular}
\end{sc}
\end{small}
\end{center}
\end{table*}

\subsection{Plotting Empirical Lipschitz Constant}
\label{sec:emlip}

In order to estimate the Lipschitz constants of LipConvNet networks and their TT-compressed counterparts, we plot histograms of empirical Lipschitz constants inspired by~\citet{Sanyal2020StableRN}.
The Lipschitz constants are evaluated by attacking each image from the test set using the FGSM attack of a fixed radius (0.5 in our experiments)~\citep{Goodfellow2015ExplainingAH}.
We compute a ratio that is upper-bounded by the true Lipschitz constant $L$:
\begin{equation} \label{eq:app_estlip}
L_{\mathrm{estimated}}(\mathcal{X}) \equiv \frac{\left \| f(\tens{X}) - f(\tens{X}_{\mathrm{attacked}}) \right \|_2}{\left \| \tens{X} - \tens{X}_{\mathrm{attacked}} \right \|_2} \leqslant L
\end{equation}
for each image and plot the results as a histogram. 
Here, $f(\tens{X})$ is the output of a model $f$ for an input image $\tens{X}$; $\tens{X}_{\mathrm{attacked}}$ is $\tens{X}$ perturbed, as described in the first paragraph.
Then, by computing $L_{\mathrm{estimated}}(\mathcal{X})$ for different $\mathcal{X}$ from the dataset, we plot the histograms (Fig.~\ref{fig:lip_ol}). 
These histograms give us insights into the true Lipschitz constant of the network.

If all network layers are $1$-Lipschitz, then we expect each $L_{\mathrm{estimated}}(\mathcal{X})$ on the histogram to be strictly bounded by~$1$.
However, for the LipConvNet architectures, the bound on the Lipschitz constant might not be exact in practice due to truncating Taylor expansion.
In our case, the frame matrices $\tens{K}^{(1)}$ and $\tens{K}^{(3)}$ are trained with the regularization loss, so they are not precisely orthogonal either.
By increasing the coefficient $\lambda_\mathrm{ort}$ of this loss component, though, we can force the matrices to become orthogonal by the end of the training process.

To support this claim, we provide the distributions of the empirical Lipschitz constants for several LipConvNet architectures trained with different $\lambda_\mathrm{ort}$.
The trained models are presented in Tab.~\ref{tab:lc_ol}.
The last metric, ``Lip.'', or Maximum Empirical Lipschitz Constant, is a lower bound on the real Lipschitz constant of the corresponding model.
It is obtained as the maximum value of all calculated empirical Lipschitz constants.
As we can see, by increasing $\lambda_\mathrm{ort}$, we can balance between the $1$-lipschitzness and the quality of metrics: the higher $\lambda_\mathrm{ort}$, the lower the metrics are and the more constrained the Lipschitz constant is.
Numbers selected with bold correspond to models that were selected as best for the corresponding LipConvNet-$N$ architectures: they demonstrate high performance in terms of accuracy and robust metrics while at the same time maintaining the Maximum Empirical Lipschitz Constant that does not exceed 1.
We also note that LipConvNet architectures do not converge if the convolutional layers are too far from being 1-Lipschitz.
For example, this effect was observed on LipConvNet-5 with $\lambda_\mathrm{ort} =$ $4e3$, where the $\lambda_\mathrm{ort}$ turns out to be not big enough.

\begin{figure}
  \centering
  \subfigure[]{\includegraphics[width=0.45\textwidth]{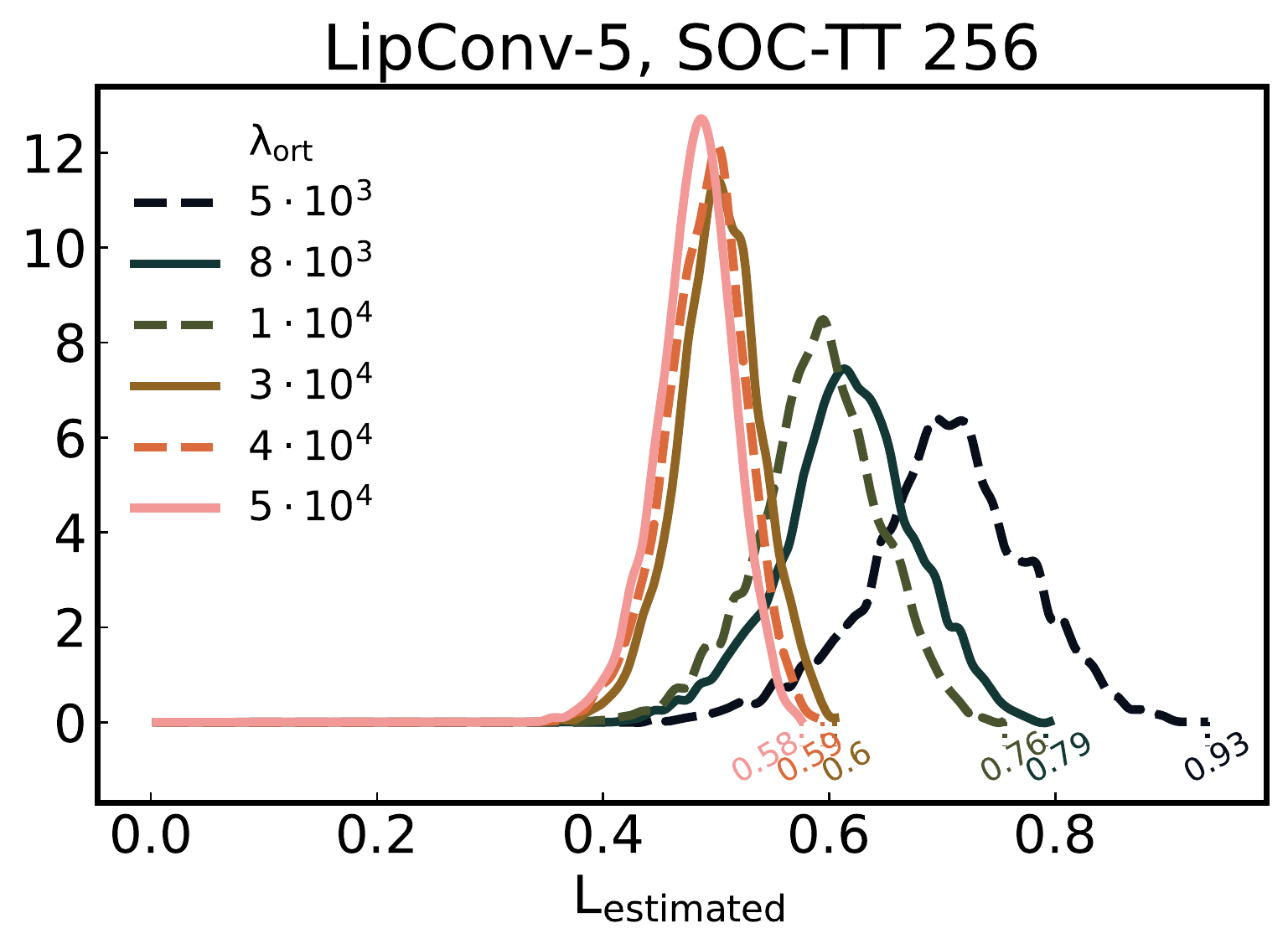}}
  \subfigure[]{\includegraphics[width=0.45\textwidth]{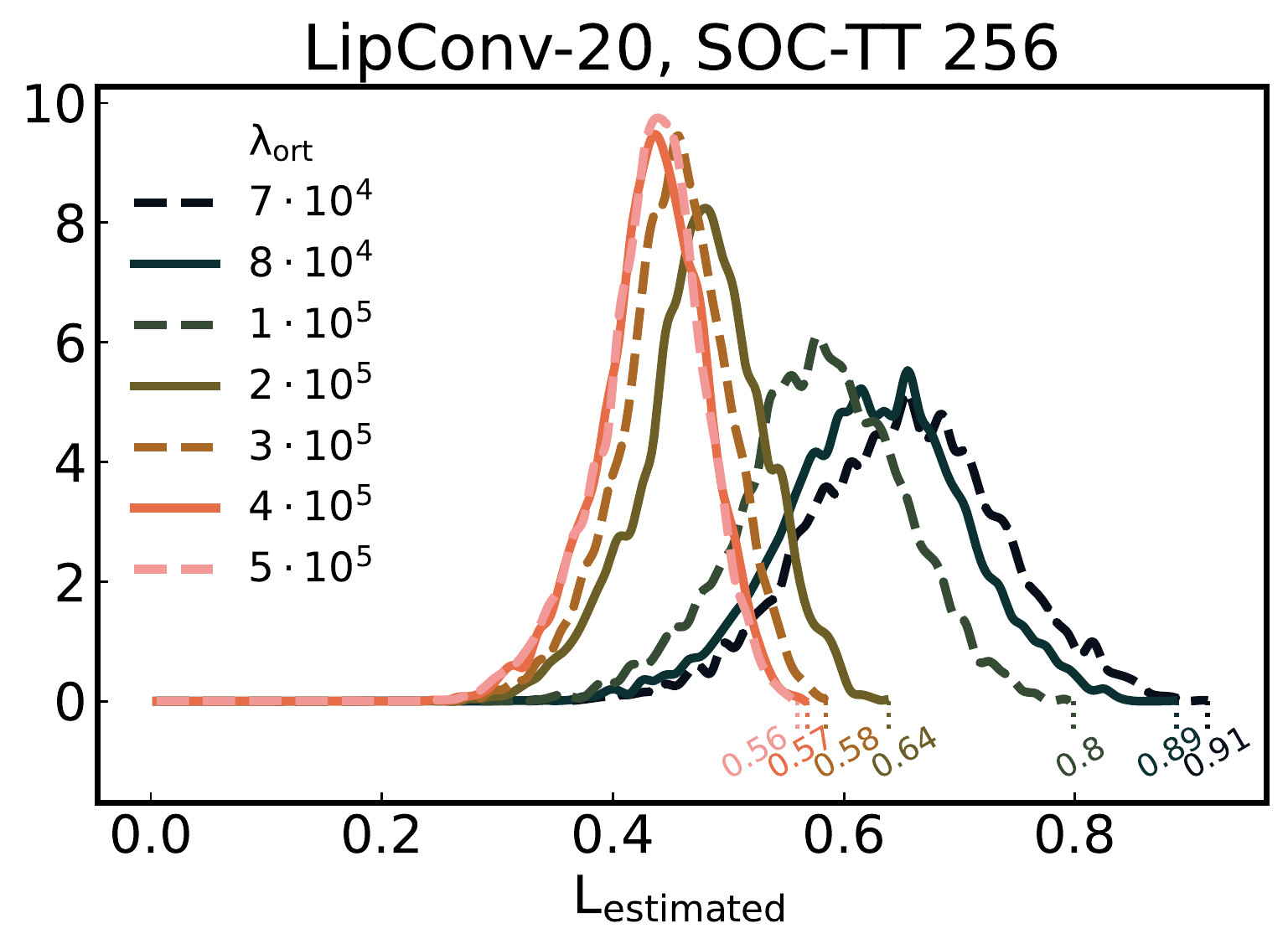}}
  \subfigure[]{\includegraphics[width=0.45\textwidth]{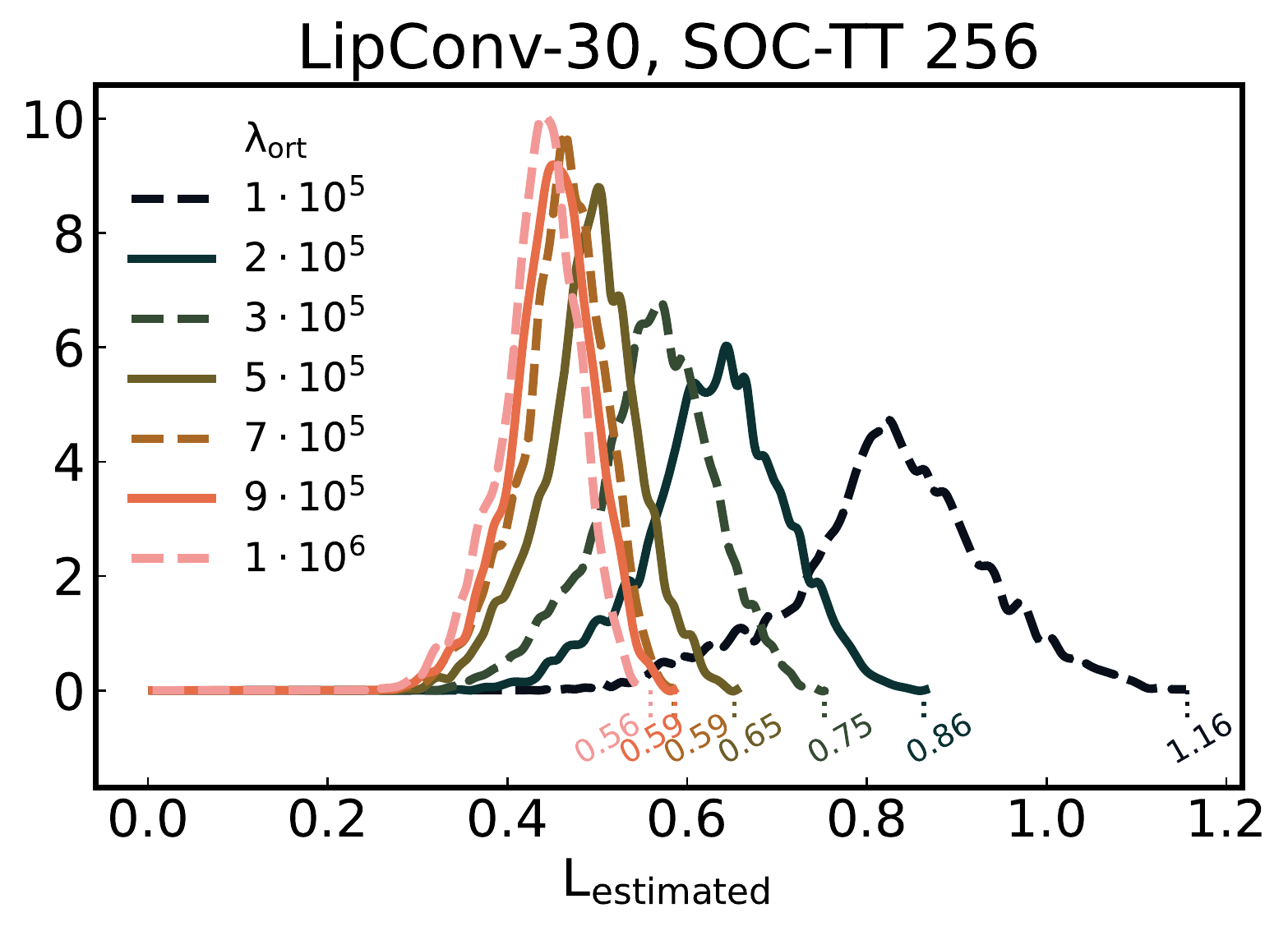}}
  \vspace{-0.1cm}
  \caption{Histograms of empirical constants~\eqref{eq:app_estlip} of networks for different regularization parameters~$\lambda_\mathrm{ort}$ and  LipConvNet-$N$ architectures.}
  \label{fig:lip_ol}
\end{figure}

Fig.~\ref{fig:lip_ol} demonstrates the shifts in distribution depending on $\lambda_\mathrm{ort}$.
We can observe ranges of empirical Lipschitz constants on the $x$-axis.
Maximum values for each model, denoted with dotted vertical lines, are the discussed Maximum Empirical Lipschitz Constants in Tab.~\ref{tab:lc_ol}.

\subsection{Inference Time of a SOC-TT Layer}

\begin{wrapfigure}{R}{0.5\textwidth}
 \vspace{-1em}
	\centering
	\includegraphics[width=0.4\textwidth]{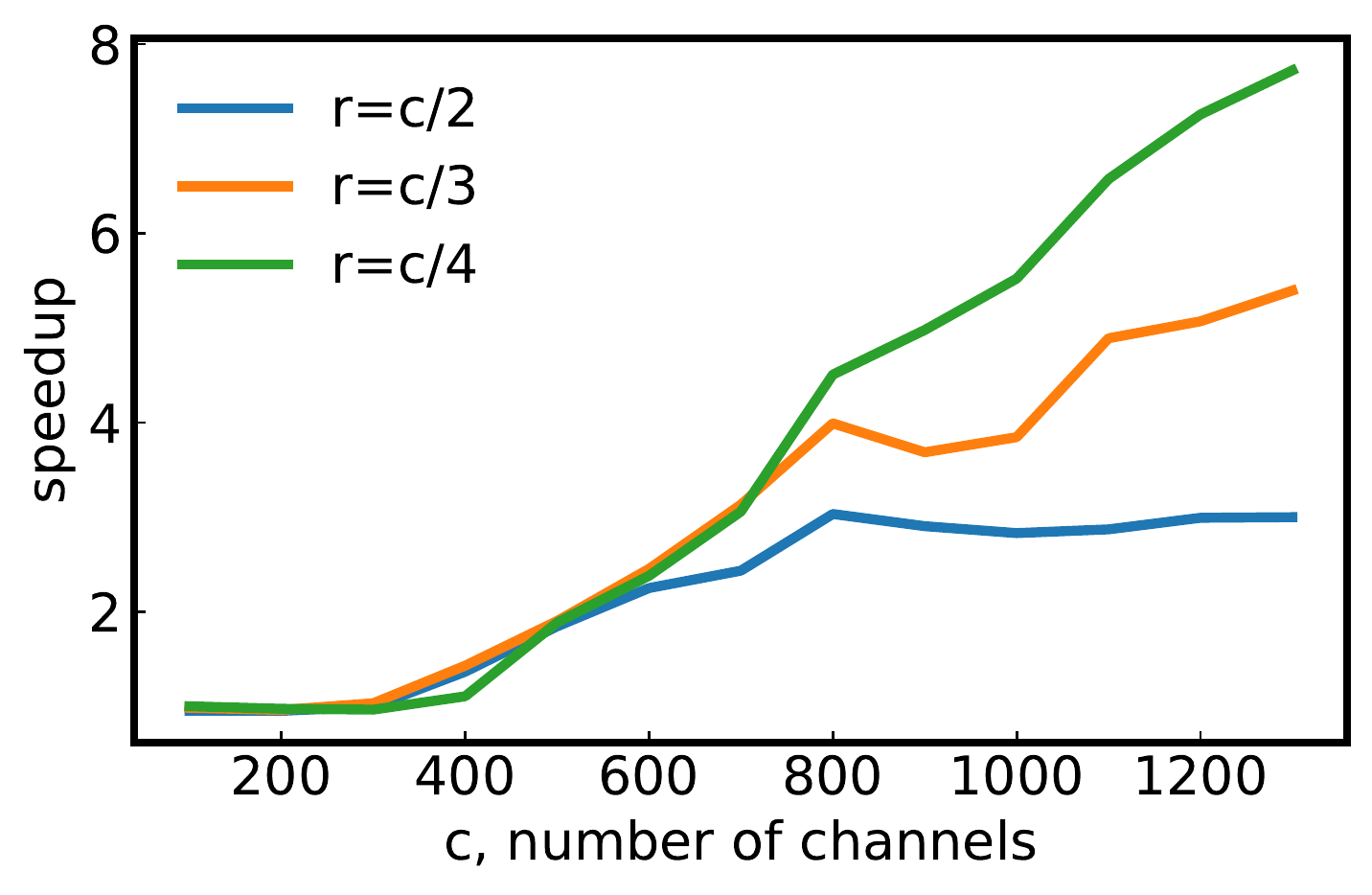}
	\caption{Speedups (w.r.t. uncompressed layer) of the application of a SOC-TT layer, $n=16$.}
	\label{fig:soc_speedup}
 \vspace{-1em}
\end{wrapfigure}

In this section, we present the inference time of our framework when applied to the SOC method.
In particular, we consider the application of a single SOC layer with various numbers of channels.
Fig.~\ref{fig:soc_speedup} illustrates speedups when a single SOC layer is accelerated using the proposed method with different rank values. 
The figure suggests that for larger residual networks that contain layers with a number of channels up to $1000$, the speedup can be up to $\approx6$ times for $c/4$ and up to $\approx3$ times for $c/2$.
However, for networks where the number of channels is less than $500$, the speedup is less than $2$.

\subsection{WideResNet16-10 on CIFAR-100}
\label{sec:appC_cifar100}
\begin{table*}[t!]
\caption{
Performance metrics for different constraints applied to WideResNet-16-10 trained on CIFAR-100.
Clipping to 1 increases the baseline performance only without TT decomposition, while the other methods (clipping to 2 and division) yield an increase in some of the metrics.
``Speedup'' is the speedup of an overhead resulting from singular values control in all the layers in the network.
Clipping the whole network w/o decomposition takes 6.2 min, while the application of division takes 0.6 sec.
``Comp.'' (compression) is the ratio between the number of parameters of convolutional layers in the original ($\sim16.8$M) and decomposed networks.
}
\label{tab:wrn16_cifar100}
\vskip 0.15in
\begin{center}
\begin{small}
\begin{sc}
\makebox[\textwidth][c]{
\begin{tabularx}{1.\textwidth}{cc|>{\centering\arraybackslash}Xc>{\centering\arraybackslash}X>{\centering\arraybackslash}Xcc}
\toprule
Method & Rank & Acc.~$\uparrow$ & AA~$\uparrow$ & CC~$\uparrow$& ECE~$\downarrow$ & \multirow{1}{*}{Clip (s)~$\downarrow$} & \multirow{1}{*}{Comp.~$\uparrow$}\\ 
\midrule
\multirow{4}{*}{Baseline} & -- & 77.97 & \textbf{27.24} & \textbf{48.95} & \textbf{8.27} & -- & 1.0\\
& 192 & 78.53 & 25.5 & 47.28 & 11.78 & -- & 3.6\\
& 256 & \textbf{78.55} & 26.11 & 47.72 & 11.76 & -- & 2.4\\
& 320 & 76.51 & 24.81 & 46.18 & 12.5 & -- & 1.8\\
\midrule
\multirow{4}{*}{Clip to 1} & -- & \textbf{78.99} & \textbf{27.84} & \textbf{47.89} & \textbf{10.09} & 1.0 & 1.0\\
& 192 & 77.7 & 25.02 & 46.77 & 11.9 & 4.1 & 3.6\\
& 256 & 77.71 & 27.15 & 46.24 & 11.88 & 3.3 & 2.4\\
& 320 & 77.79 & 26.34 & 47.51 & 11.75 & 2.3 & 1.8\\
\midrule
\multirow{4}{*}{Clip to 2} & -- & \textbf{79.92} & \textbf{28.15} & \textbf{47.73} & \textbf{9.4} & 1.0 & 1.0\\
& 192 & 78.74 & 27.21 & 46.08 & 11.41 & 4.1 & 3.6\\
& 256 & 79.39 & 27.54 & 47.26 & 11.07 & 3.3 & 2.4\\
& 320 & 79.52 & 26.82 & 47.68 & 10.43 & 2.3 & 1.8\\
\midrule
\multirow{4}{*}{Division} & -- & 78.59 & \textbf{26.85} & \textbf{48.38} & 10.97 & 1.0 & 1.0 \\
& 192 & 78.65 & 25.4 & 47.86 & 8.29 & 4.1 & 3.6 \\
& 256 & \textbf{78.98} & 26.26 & 47.77 & \textbf{8.21} & 3.3 & 2.4 \\
& 320 & 77.58 & 25.69 & 46.28 & 12.18 & 2.3 & 1.8 \\
\bottomrule
\end{tabularx}
}
\end{sc}
\end{small}
\end{center}
\vskip -0.1in
\end{table*}

In addition to evaluating our framework on CIFAR-10, we conducted experiments on another classic vision dataset, CIFAR-100.
We trained WideResNet-16-10 on CIFAR-100 with the same experimental setup as in \citet{gouk2020regularisation} and compared clipping and division as in the main body of the paper.
The results are presented in Tab.~\ref{tab:wrn16_cifar100}.
Compared with CIFAR-10 experiments from~\ref{tab:wrn16}, we observe that both clipping and division do not give gain on CIFAR-C (CC).
The other metrics tend to improve with the control of singular values.
Except for clipping-$1$, the TT-compressed versions with singular value control outperform the baseline accuracy in most cases.
The AA metric of the baseline is outperformed when using clipping to $2$ and $r=256$.
\end{document}